\documentclass[11pt,letterpaper, logo]{mystyle}

\usepackage[utf8]{inputenc} 
\usepackage[T1]{fontenc}    
\usepackage[bottom]{footmisc} 
\usepackage{hyperref}       
\usepackage{url}
\usepackage[percent]{overpic}            
\expandafter\def\expandafter\UrlBreaks\expandafter{\UrlBreaks\do\-} 
\usepackage{booktabs}       
\usepackage{amsfonts}       
\usepackage{microtype}      
\usepackage{lipsum}         
\usepackage{graphicx}
\usepackage{enumitem}
\usepackage[numbers,sort]{natbib}
\usepackage{titletoc}     
\usepackage{float}        
\usepackage{amsmath}
\usepackage{multirow}
\usepackage{makecell}
\usepackage{tabularx}

\usepackage{subcaption}
\usepackage{wrapfig}
\usepackage{algorithm}
\usepackage{algorithmic}

\newtheorem{theorem}{Theorem}
\newtheorem{assumption}[theorem]{Assumption}

\newtheorem{lemma}[theorem]{Lemma}
\newtheorem{corollary}[theorem]{Corollary}
\newtheorem{proposition}[theorem]{Proposition}
\newtheorem{remark}[theorem]{Remark}

\usepackage{tcolorbox}
\tcbuselibrary{breakable}
\usepackage{tikz}
\usetikzlibrary{calc}
\definecolor{remarkblue}{RGB}{31,78,140}
\definecolor{gaingreen}{RGB}{25,107,36}
\newcommand{\gain}[1]{\,{\footnotesize\bfseries\textcolor{gaingreen}{$\uparrow$\textbf{#1}}}}
\definecolor{dropred}{RGB}{165,49,49}
\newcommand{\drop}[1]{\,{\footnotesize\textcolor{dropred}{$\downarrow$#1}}}
\newtheorem*{restatethm}{Theorem}
\newtcolorbox{thmbox}{breakable, colback=remarkblue!7,
  colframe=remarkblue!90, boxrule=1.0pt, arc=2pt,
  left=7pt, right=7pt, top=4pt, bottom=4pt,
  before skip=6pt, after skip=6pt}
\newtcolorbox{remarkbox}{breakable, colback=remarkblue!5,
  colframe=remarkblue!85, boxrule=0.7pt, arc=2pt,
  left=7pt, right=7pt, top=4pt, bottom=4pt,
  before skip=6pt, after skip=6pt}
\newtcolorbox{promptbox}[1]{breakable, colback=gray!4,
  before upper={\setlength{\parindent}{0pt}},
  colframe=gray!55!black, boxrule=0.6pt, arc=3pt,
  left=7pt, right=7pt, top=5pt, bottom=5pt,
  title={#1}, fonttitle=\small\bfseries,
  colbacktitle=gray!15, coltitle=black,
  before skip=8pt, after skip=8pt}

\newcommand{\ours}{WMRL}

\graphicspath{ {./figures/} }

\setlist[itemize]{leftmargin=12pt}

\runningtitle{Scaling Automatic Research Agents via World Models}

\title{
\textbf{Scaling Automatic Research Agents via World Models}
}

\date{\vspace{-3ex}}

\makeatletter
\def\thm@space@setup{\thm@preskip=3pt \thm@postskip=3pt}
\makeatother

\begin{document}

\author{
    \normalfont
    \begin{minipage}{0.98\linewidth}
    \centering
    \setlength{\parskip}{0pt}
    \textbf{Xiyuan Yang}$^{1\dagger}$,
    \textbf{Sheikh Sarwar}$^{2}$,
    \textbf{Jingru Cheng}$^{2}$,
    \textbf{Zhan Shi}$^{2}$,
    \textbf{Duanshun Li}$^{2}$,
    \par\vspace{2pt}
    \textbf{Huiyuan Chen}$^{2}$,
    \textbf{Haiyang Zhang}$^{2}$,
    \textbf{Xing Fan}$^{2}$,
    \textbf{Chenlei Guo}$^{2}$,
    \textbf{Jingrui He}$^{1*}$,
    \textbf{Zhenyu Liao}$^{2*}$
    \par\vspace{7pt}
    $^{1}$University of Illinois Urbana-Champaign \quad
    $^{2}$Amazon \quad
    $^{*}$Corresponding authors
    \end{minipage}
}

\begin{abstract}

\textbf{\large Abstract:}
Automating empirical research is a long-standing direction of AI. Recent automatic research (AutoResearch) agents bring this goal within reach, as modern LLMs show the capability to independently implement solutions and learn from the execution outcomes. 
Behind these gains, post-training (especially RL) plays a central role. 
In this paper, we identify a fundamental tension when scaling RL for these agents: the two components of every AutoResearch trajectory (agent generation and environment execution) scale in very different manners, 
since all generation shares compute through batching, while each execution occupies its exclusive sandbox and real machine time.
As a result, the environment execution dominates the training cost and becomes the bottleneck as trajectories grow.
To resolve this tension, we propose \emph{World Model RL} (\ours{}), which replaces environment execution with a world model to remove this bottleneck. 
Additionally, the world model can be imperfect, as its rewards are corrupted by bias and noise. Therefore, we further equip WMRL with two mitigations, \emph{Online Debiasing} and \emph{Inverse-Variance Denoising}, which offset the bias and suppress the noise respectively.
Theoretically, we prove that both mitigations of \ours{} strictly improve the convergence guarantee.
Empirically, WMRL accelerates training by $3$--$4\times$ on various tasks at different agent scales, while exceeding the performance of standard RL baselines. Moreover, our post-trained 4B and 9B agents outperform much larger open-weight agents of 48B and 120B on held-out benchmarks. Beyond AutoResearch, WMRL also transfers to post-training embodied VLA policies, which demonstrates the generalizability of our method.

\par\vspace{9pt}\noindent
\normalfont\textbf{Project Page:}\ \url{https://xiyuanyang45.github.io/WMRL/}

\end{abstract}
\maketitle
{\let\thefootnote\relax\footnotetext{$^{\dagger}$Work done during an internship at Amazon.}}

\section{Introduction}
\label{sec:introduction}

\begin{figure}[b!]
    \centering
    \includegraphics[width=\textwidth]{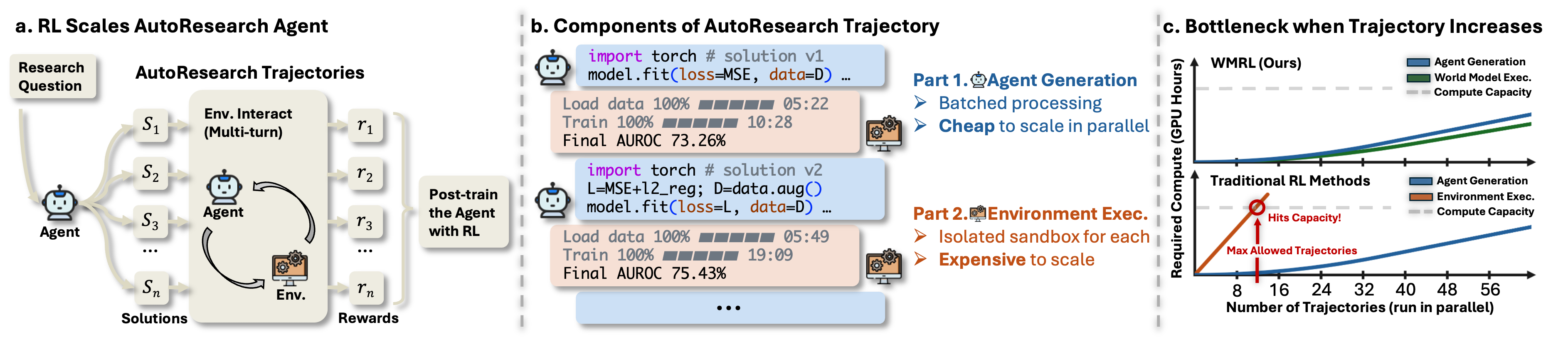}
    \caption{
    \textbf{AutoResearch trajectories scale asymmetrically, making environment execution the bottleneck, and our \ours{} removes it.}
    \textbf{(a)} Per research question, the agent proposes a group of solutions $S_i$, graded by execution into rewards $r_i$, and RL demands a massive number of such trajectories.
    \textbf{(b)} Per trajectory, generation amortizes compute via batching, but execution needs one isolated sandbox per solution.
    \textbf{(c)} Execution compute thus hits capacity first in traditional RL, whereas \ours{} scales without this limit.
    }
    \label{fig:teaser}
\end{figure}

An Automatic Research (AutoResearch) agent is a language model that independently conducts empirical research~\citep{lu2024ai,gottweis2025aicoscientist,schmidgall2025agentlab}. 
Given a research question, it formulates an idea, implements the experiment, analyzes the outcome, and iterates~\citep{yao2023react,yao2023tot,shinn2023reflexion}. 
Such agents have already proven capable across various domains. In the natural sciences, for example, they design chemical syntheses and propose reaction conditions that are validated by wet-lab experiments later~\citep{boiko2023autonomous,bran2024chemcrow,szymanski2023alab}; in machine learning and data science, they explore real datasets and build training pipelines that outperform human experts~\citep{jiang2025aide,guo2024dsagent,huang2024mlagentbench}.
In an AutoResearch task, the agent iteratively generates and executes solutions. These interactions form trajectories, and the execution outcomes provide rewards.
With both trajectories and rewards in place, AutoResearch is a natural fit for reinforcement learning (RL)~\citep{schulman2017ppo,shao2024deepseekmath,lambert2024tulu3,guo2025deepseekr1}, a promising direction to further improve these capabilities.

Like most of RL's successes, training strong AutoResearch agents demands scale, i.e., massive online trajectories collected during training~\citep{guo2025deepseekr1}.
However, we identify a fundamental tension underneath this demand: the two components of a trajectory (i.e., agent generation and environment execution) scale in different manners as the trajectory volume grows.
On the generation side, rollouts are served by modern inference backends (e.g., vLLM~\citep{kwon2023vllm} and SGLang~\citep{zheng2024sglang}), 
where batching techniques allow concurrent trajectories to share compute, making the cost of additional trajectories negligible.
In contrast, execution cost cannot be amortized in this way. Every candidate solution must run in an isolated sandbox that loads the data and trains models on real GPUs~\citep{qiang2025mledojo,sun2026sweworld}, so each additional trajectory incurs the full cost, and the total cost grows linearly with the number of trajectories.
This asymmetry makes environment execution the dominant bottleneck when scaling RL for AutoResearch agents.
Such a bottleneck motivates two research questions:

\begin{center}
\textbf{(i)} \emph{Can we replace the expensive environment (bottleneck) with a fast and scalable signal?}\\[2pt]
\textbf{(ii)} \emph{There is no free lunch, so what does this signal cost? (And how do we pay the bill?)}\\
\end{center}

The answer to question~\textbf{(i)} is to introduce a world model~\citep{ha2018world,hafner2025mastering}, which takes the environment information and the agent solution as input, and produces a simulation of real execution result as output. 
Such a simulation (though potentially flawed) takes only a few forward passes without real execution, so it batches and amortizes across trajectories just like generation. In RL training, we let the agent interact with the world model instead of the real environment, and receive rewards from the predicted outcomes.
As a result, the expensive execution now scales as gracefully as generation and no longer bounds the scale of RL.

We answer question~\textbf{(ii)} by modeling the world model's imperfection as a bias and a noise on the reward, and tracing both through the policy gradient into the convergence bound.
Specifically, we consider the world model's output as an erroneous signal that deviates from real execution by  a bias term $b$ with $|b| \le B$ and a zero-mean noise $\xi$ with standard deviation $\sigma$.
As shown in Theorem~\ref{thm:wm-conv}, the bias and the noise hinder
the final convergence through two extra error terms of $O(B^{2})$ and $O(\sigma^{2})$ respectively.
To reduce both terms, we introduce an \emph{Online Debiasing} mechanism that recasts the world model outputs to offset the bias, and an \emph{Inverse-Variance Denoising} mechanism to minimize the variance; we further ground both mechanisms in Theorem~\ref{thm:ours-conv} with a strictly improved convergence guarantee.

Our contributions are summarized as follows:
\begin{itemize}[leftmargin=*]
    \item We introduce the world model into the RL training of AutoResearch agents, which replaces the expensive environment execution with a few forward passes. It removes the critical scaling bottleneck and accelerates overall training by $3$--$4\times$ (Section~\ref{subsec:wm-env}).

    \item We design two correction mechanisms, Online Debiasing and Inverse-Variance Denoising, to counteract the bias and the noise of the world model. With both mechanisms, the accelerated training matches or even exceeds the performance of training with the real environment (Section~\ref{subsec:anchor}).

    \item We theoretically ground the entire framework. An imperfect world model introduces two error terms into the convergence bound (Section~\ref{subsec:theory-cost}), and our two mechanisms provably reduce both, yielding a strictly improved convergence guarantee (Section~\ref{subsec:theory-fix}).

    \item Through experiments on various AutoResearch tasks across two agent scales, we validate both the efficiency and the performance gains (Section~\ref{subsec:exp-main}); we further extend our method to VLA post-training tasks to demonstrate its generalizability (Section~\ref{subsec:exp-vla}).
\end{itemize}

\section{Related Work}
\label{sec:related_work}

\paragraph{AutoResearch agents and their post-training.}
Language model agents now carry out substantial parts of empirical
research, from autonomous chemical experimentation~\citep{boiko2023autonomous}
and open-ended scientific discovery~\citep{lu2024ai} to machine
learning engineering, where agents explore the space of solution
code~\citep{jiang2025aide,novikov2025alphaevolve,yamada2025aiscientistv2}, curate data as
agentic data scientists~\citep{kulikov2026autodata,bertran2026many}, and pursue
recursive self-improvement~\citep{yang2026frontis,recursive2026firststeps}.
A line of benchmarks measures this
ability on real Kaggle-style competitions, including
MLE-Bench~\citep{chan2024mlebench}, its interactive superset
MLE-Dojo~\citep{qiang2025mledojo}, DSBench~\citep{jing2024dsbench}, and
MLGym~\citep{nathani2025mlgym}, and software engineering agents follow the
same execution-driven recipe on real repositories~\citep{jimenez2024swebench,yang2024sweagent,wang2025openhands},
with training environments built at scale~\citep{pan2025swegym}. Beyond prompting, RL post-training further improves such agents, typically with
group-relative objectives~\citep{shao2024deepseekmath} served by
large-scale RL systems~\citep{sheng2024hybridflow}, and synthetic
tasks can scale the training corpus~\citep{xie2026quest}. In these
pipelines the reward comes from executing the agent's solution,
so training inherits the execution bottleneck of
Section~\ref{sec:introduction}, the problem we address.
Such advances refine how rewards are consumed~\citep{yu2025dapo,qu2026pope}, whereas we change
where they come from, so the two are orthogonal and compose.

\paragraph{Learned reward signals and their errors.}
Replacing an expensive ground truth with a learned signal is a
recurring pattern, from reward models in RLHF~\citep{ouyang2022training}
to LLM judges of model outputs~\citep{zheng2023judging}, and world
models that simulate an environment for control have long powered
model-based RL~\citep{ha2018world,bruce2024genie,hafner2025mastering}. Such world models are
now scaled to foundation models of the physical world~\citep{nvidia2025cosmos} or
written as executable code by LLMs~\citep{dainese2024codewm,tang2024worldcoder},
and learned surrogates of the execution environment have recently been
explored for software agents~\citep{sun2026sweworld}. The errors of
such proxies are equally well documented, as optimizing against an
imperfect reward model degrades true performance once the proxy is
over-trusted~\citep{gao2023scaling,coste2024rewardensembles}, and the theory of SGD with biased
gradients shows that a systematic gradient error puts a floor on
convergence that no amount of training removes~\citep{ajalloeian2020convergence}.
Prior remedies mostly recalibrate the proxy offline~\citep{zadrozny2002calibration,guo2017calibration}
or constrain the policy from exploiting it. In off-policy RL, a related
line corrects the distribution mismatch between the replay buffer and
the target policy with learned correction ratios~\citep{nachum2019dualdice,li2021offpolicy},
a finer object than the reward statistics we treat. We instead keep a small
stream of ground truth inside the training loop, correct the bias and
the variance of the proxy online, and quantify both corrections
directly in the convergence bound.

\section{Method}
\label{sec:method}

In this section, we introduce our \ours{} as follows. First, we formalize the standard RL training procedure and define the notation (Section \ref{subsec:prelim}). 
Next, we answer the question~\textbf{(i)} by incorporating the world model into the RL pipeline as a fast and scalable signal (Section \ref{subsec:wm-env}). 
We then answer the question~\textbf{(ii)} by revealing the additional error terms in the RL convergence, and then providing our mitigations (Section \ref{subsec:anchor}), which offset these error terms introduced by the world model.

\subsection{Preliminaries}
\label{subsec:prelim}

Typically, an AutoResearch task provides the agent with the research question, data format and the output requirements. 
With this context, the agent policy $\pi_\theta$ interacts with an environment (e.g., an isolated Docker container)~\citep{qiang2025mledojo,jing2024dsbench} over multiple turns to iteratively refine its solution until reaching a predefined criterion.
Such interactions form a trajectory $\tau$, and the solution receives a score $r(\tau)\in[0,1]$ that measures its quality.

To enhance the agent policy, a wide range of works leverage RL to maximize the expected score $J(\theta) := \mathbb{E}_{\tau\sim\pi_\theta}[r(\tau)]$ with group-relative policy optimization (GRPO)~\citep{shao2024deepseekmath}\footnote{We adopt GRPO as the most common objective for agentic RL at scale~\citep{sheng2024hybridflow,guo2025deepseekr1}, while the specific RL objective is orthogonal to our contribution, which acts mostly on the reward side rather than on the update rule.} as follows:
For each task, GRPO samples a group of $n \ge 2$ independent trajectories
$\tau_1,\dots,\tau_n$ in parallel, and then gives the grade $r(\tau_i)$ of each trajectory by executing the final solution in the environment.
The scores are then normalized within the group into advantages $A_i := r(\tau_i) - \frac{1}{n}\sum_{j} r(\tau_j)$, which yield the gradient estimate as:
\begin{equation}
\hat g \;:=\; \frac{1}{n}\sum_{i=1}^{n} A_i\, s_i,
\qquad s_i := \nabla_\theta \log \pi_\theta(\tau_i),
\label{eq:grpo}
\end{equation}
where $s_i$ is the raw policy gradient, and the parameter is updated by $\theta_{t+1} = \theta_t + \gamma\,\hat g_t$ over $T$ steps with step size $\gamma$.
Throughout training, every step costs online rewards, and each of the rewards requires executing a solution in a new environment with exclusive GPU assignments.

\begin{figure}[t]
\centering
\begin{overpic}[width=\textwidth]{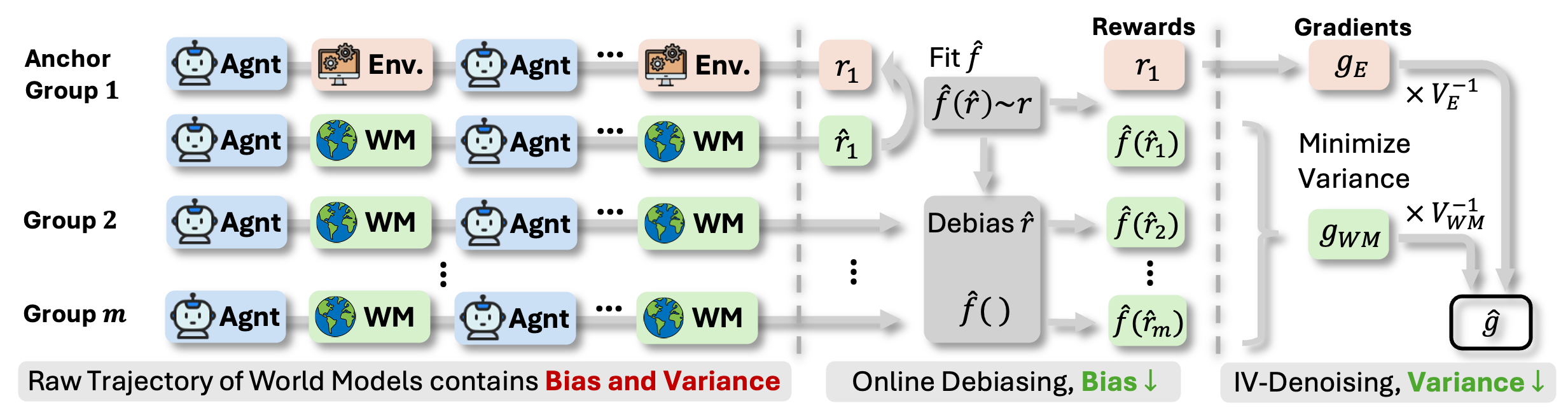}
  \put(62.7,9.7){\makebox(0,0){\scriptsize\hyperref[eq:isotonic]{Eq.~(\ref*{eq:isotonic})}}}
  \put(87.4,6.8){\makebox(0,0){\scriptsize\hyperref[eq:ivw]{Eq.~(\ref*{eq:ivw})}}}
\end{overpic}
\caption{
\textbf{\ours{} corrects the world model rewards in two steps.} Each
row is one of the $m$ groups in a batch. Every group is
graded by the world model into biased, noisy scores $\hat r$, and
anchor groups are also graded by real execution. The monotone map
$\hat f$ is fit on the score pairs (\emph{Online Debiasing}) to remove the bias, and $g_E$ and $g_{WM}$ are fused
by \emph{Inverse-Variance Denoising} to lower the update variance.
}
\label{fig:method}
\end{figure}

\subsection{World Model as an Environment}
\label{subsec:wm-env}

To provide fast and scalable environment signals (asked in question~\textbf{(i)}), we introduce a world model to replace the expensive execution, which constitutes the backbone of \ours{}.
In our setting, we adopt as the world model a general language model that simulates the environment (with potential errors), instantiated with the same backbone as the agent and queried for the execution outcome.
It takes the same task context and an agentic solution as input, and predicts the execution outcome as output, from which an estimated score $\hat r(\tau_i)$ is read off in place of the true $r(\tau_i)$. 
Since this interface is identical to that of the real environment, the pipeline in Section~\ref{subsec:prelim} runs unchanged with the estimated scores as:
\begin{equation}
\hat A_i \;:=\; \hat r(\tau_i) - \frac{1}{n}\sum_{j} \hat r(\tau_j),
\label{eq:wm-grpo}
\end{equation}
and the gradient estimate of Equation~\eqref{eq:grpo} is computed with $\hat A_i$ in place of $A_i$.
In effect, the parameters now ascend the surrogate objective $\mathbb{E}_{\tau\sim\pi_\theta}[\hat r(\tau)]$ rather than $J(\theta)$.
As a result, the execution part of a trajectory now scales as gracefully as the generation side, which directly removes the execution bottleneck.
However, we note that this replacement is not free, as an imperfect predicted outcome can deviate from the real execution outcome.
We characterize and mitigate such deviation in Section~\ref{subsec:anchor} next.

\subsection{Anchor Signal for Error Correction}
\label{subsec:anchor}

We now answer question~\textbf{(ii)}. For any world model, the deviation of its predicted score can be decomposed into a systematic part and a random part as follows:
\begin{equation}
\hat r(\tau) \;=\; r(\tau) + b(\tau) + \xi(\tau),
\label{eq:error-structure}
\end{equation}
where the bias $b(\tau) := \mathbb{E}[\hat r(\tau) \mid \tau] -
r(\tau)$ collects the systematic part that survives averaging, and the noise
$\xi(\tau)$ is the zero-mean remainder. Since the scores are naturally bounded, we write
$B := \sup_{\tau} |b(\tau)|$ for the maximal magnitude of the bias and
$\sigma := \sup_{\tau} \mathrm{std}(\xi(\tau))$ for the maximal
standard deviation of the noise. 
The remark below summarizes how the two parts affect the convergence.

\begin{remarkbox}
\begin{remark}[The price of the world model replacement]
\label{rem:price}
With real execution rewards, standard RL converges with a bound
$J^\star - \mathbb{E}[J(\theta_T)] \le \varepsilon(T)$ on the gap to
the optimal score $J^\star := \sup_\theta J(\theta)$. With
world model rewards, the same analysis gives
\begin{equation*}
J^\star - \mathbb{E}[J(\theta_T)] \;\le\; \varepsilon(T)
\;+\; O({\color{dropred}B^{2}}) \;+\; O({\color{dropred}\sigma^{2}}),
\end{equation*}
where $\varepsilon(T)$ is the standard convergence term
(Theorem~\ref{thm:wm-conv}).
We next resolve the $O(B^{2})$ and $O(\sigma^{2})$ terms.
\end{remark}
\end{remarkbox}

Our goal is to remove or reduce the two additional errors (the $O(B^{2})$ and $O(\sigma^{2})$ term) in
Remark~\ref{rem:price}. Removing them starts with estimating them, and
by Equation~\eqref{eq:error-structure} both are defined against the
true score $r(\tau)$, which the world model itself does not provide.
Since the truth can only come from environment execution, we take a small step back and reintroduce a marginal amount of execution for this estimation (which we prove to be enough in Section~\ref{sec:theory}).
\ours{} therefore keeps a thin stream of ground truth during training, which we call the \emph{anchor signal}. A small fraction of the groups (\emph{anchor groups}) is graded by both the world model and real execution,\footnote{In practice, anchor groups make up about $10\%$ of all groups.} and the resulting score pairs $\mathcal{P} = \{(\hat r_j, r_j)\}$ feed the two mechanisms below (Figure~\ref{fig:method}).

We first correct the \textbf{bias} by \emph{Online Debiasing}, a
monotone distribution recalibration.
The training score pairs reveal how the world model score drifts from the
truth, so we fit a mapping function over all of them,
\begin{equation}
\hat f \;=\; \operatorname*{arg\,min}_{f}\;
\sum_{(\hat r_j,\, r_j) \in \mathcal{P}}
\big(f(\hat r_j) - r_j\big)^{2},
\label{eq:isotonic}
\end{equation}
where $f$ ranges over monotone functions solved by isotonic regression~\citep{zadrozny2002calibration}. 
With this mapping, all world model scores are then recast into $\hat f(\hat r(\tau_i))$ before the advantages are formed, and $\hat f$ is refit as new pairs arrive on each step to track the drift over training. 

For the \textbf{noise} part, since it cannot be estimated pointwise, we
turn to suppress it by \emph{Inverse-Variance Denoising}, which fuses
the two reward streams to reach a lower variance. In each step the $m$
group indices split into the anchor set $\mathcal{G}_E$ and the world
model set $\mathcal{G}_{WM}$, and every group $k$ gives an estimate
$\hat g_k$ of the same policy gradient, with variance $V_{E}$ or
$V_{WM}$ by how it was graded. The question is therefore how to
combine them, as one stream is scarce but free of world model noise and
the other abundant but noisy. Following Lemma~\ref{lem:ivw}, the
minimal-variance combination weights each stream by the inverse of its
variance and attains a variance strictly below either stream alone, so
we combine them as
\begin{equation}
\hat g \;=\;
\frac{\ V_{E}^{-1}\, g_E \;+\; V_{WM}^{-1}\, g_{WM}\ }
{\ V_{E}^{-1} |\mathcal{G}_E|
\;+\; V_{WM}^{-1} |\mathcal{G}_{WM}|\ },
\qquad
g_E := \!\!\sum_{k \in \mathcal{G}_E}\!\! \hat g_k,
\quad
g_{WM} := \!\!\sum_{k \in \mathcal{G}_{WM}}\!\! \hat g_k,
\label{eq:ivw}
\end{equation}
where the stream gradients $g_E$ and $g_{WM}$ sum the group
estimates $\hat g_k = \frac{1}{n}\sum_{i=1}^{n} A_{k,i}\, s_{k,i}$,
each applying Equation~\eqref{eq:grpo} to group $k$ with the true
scores on $\mathcal{G}_E$ and the calibrated scores on
$\mathcal{G}_{WM}$.

Equation~\eqref{eq:ivw} is the conceptual form of the fusion, and it
attains the minimal variance once $V_{E}$ and $V_{WM}$ are given. In
implementation, Equation~\eqref{eq:ivw} reduces to the simple form
$\hat g = \big(\rho\, g_E + g_{WM}\big)\big/\big(\rho\,|\mathcal{G}_E| +
|\mathcal{G}_{WM}|\big)$, with only a single unknown left, the ratio
$\rho := V_{WM}/V_{E}$. To derive this ratio, we first show in
Section~\ref{sec:theory} that $\rho = 1 + c\,\sigma^{2}$, in which the
reward noise $\sigma^{2}$ is measurable and the constant $c$ is
estimable. We then measure $\sigma^{2}$
by the mean squared residual $\hat\eta^{2}$ between the calibrated and
the true scores on the anchor groups, and estimate $c$ by the
reciprocal of one such residual recorded at the end of a warmup phase,
$c := 1/\hat\eta^{2}_{\mathrm{cal}}$. The anchor groups then enter
Equation~\eqref{eq:ivw} with the weight $\rho = 1 +
\hat\eta^{2}/\hat\eta^{2}_{\mathrm{cal}}$ and the world model groups
with the weight one, a rule that carries no free parameter. We further
justify the optimality of this estimate both theoretically and
empirically in Remark~\ref{rem:plugin}, as the leading-order term of the estimation error vanishes exactly by construction and the remaining higher-order excess measures a few percent at most in our runs.
Both mitigations are validated empirically in Section~\ref{sec:experiment} and analyzed in Section~\ref{sec:theory}.

\subsection{Discussion}
\label{subsec:discussion}

We note that the proposed denoising weights are not only theoretically sound
(Section~\ref{sec:theory}) but also empirically interpretable. On the
anchor groups, the tracked residual $\hat\eta^{2}$
measures how far the calibrated predictions fall from the true scores,
and thereby audits the world model throughout training.
When the audit worsens, the anchor weight in
Equation~\eqref{eq:ivw} rises. Consider a task whose outcome hinges on
randomness that no reading of the solution reveals. In this case, $\hat\eta^{2}$ remains large and the gradient is
carried by the anchor groups, so \ours{} degrades toward the standard
GRPO of Section~\ref{subsec:prelim} instead of learning from a
corrupted signal. This fallback follows from the measurement itself rather than a
mixing ratio fixed in advance.

Additionally, although the design targets the execution bottleneck of AutoResearch,
it is not limited to this scenario. The construction requires only that rewards are expensive to execute yet predictable
from the artifacts the agent produces, and that a small stream of
ground truth stays available for anchoring. Post-training embodied
policies is one such instance, where real rollouts are slow while
learned simulators are cheap, and we validate this transfer on VLA
tasks in Section~\ref{sec:experiment}.

\section{Theoretical Analysis}
\label{sec:theory}

In this section, we first analyze the additional error terms that an imperfect world model introduces, and then justify our mitigation by a convergence analysis. 
We first bound the convergence of training on world model rewards alone in Section~\ref{subsec:theory-cost} (Theorem~\ref{thm:wm-conv}), and then bound the convergence of \ours{} in Section~\ref{subsec:theory-fix} (Theorem~\ref{thm:ours-conv}), and compare the two term by term. We put the full proofs and auxiliary lemmas in Appendix~\ref{app:proofs}.

\subsection{Setup}
\label{subsec:theory-setup}

We analyze the RL training procedure of Section~\ref{subsec:prelim},
namely $T$ ascent steps with step size $\gamma$ on the gradient
estimator of Equation~\eqref{eq:grpo}, and we measure progress by the expected gap
to the optimal score, $J^\star - \mathbb{E}[J(\theta_T)]$ with
$J^\star := \sup_\theta J(\theta)$ and $\Delta_0 := J^\star -
J(\theta_0)$. The analysis rests on one standard assumption.

\begin{assumption}[Regularity]
\label{ass:reg}
$J$ is $L$-smooth and satisfies the gradient domination condition
$\|\nabla J(\theta)\|^{2} \ge 2\mu\,(J^\star - J(\theta))$ for some
$\mu > 0$, and the log-likelihood gradient is bounded by
$\|\nabla_\theta \log \pi_\theta(\tau)\| \le M$ for all $\tau$.
\end{assumption}

Everything the world model contributes enters through the two
quantities of Section~\ref{subsec:anchor}, the bias magnitude $B$ and
the noise standard deviation $\sigma$ of
Equation~\eqref{eq:error-structure}, and through the two group
variances $V_{E}$ and $V_{WM}$ of
Equation~\eqref{eq:ivw}. The gradient estimator of
Equation~\eqref{eq:grpo} is linear in the rewards, so the two error
parts act on the gradient separately. The bias is deterministic and
shifts its mean. The noise is zero-mean and only inflates the variance
of a world model group, to $V_{WM} = (1 + c\,\sigma^{2})\,V_{E}$ with
$c = 4M^{2}/(n V_{E})$, as announced in Section~\ref{subsec:anchor}. Only the product $c\,\sigma^{2}$ enters
the bounds, and Section~\ref{subsec:anchor} supplies it by measuring
its varying part and normalizing its scale, so neither is needed on
its own.
Throughout we take $\gamma \le 1/(8L)$.

\subsection{The Cost of World Model Rewards}
\label{subsec:theory-cost}

We first consider the setting of Section~\ref{subsec:wm-env}, where
every group is graded by the world model and no real execution takes
place, so the rewards carry the bias and the noise of
Equation~\eqref{eq:error-structure} at every step. Since the gradient estimator
built on Equation~\eqref{eq:wm-grpo} is linear in the rewards, the two
deviations reach the gradient in different ways, the bias shifting its
mean and the noise inflating its variance, and they therefore enter the
convergence bound as the two separate terms announced in
Remark~\ref{rem:price}.

\begin{thmbox}
\begin{theorem}[Convergence with world model rewards]
\label{thm:wm-conv}
Under Assumption~\ref{ass:reg}, GRPO trained on world model rewards
satisfies
\begin{equation*}
J^\star - \mathbb{E}[J(\theta_T)] \;\le\;
\Big(1-\tfrac{\gamma\mu}{4}\Big)^{T}\Delta_0
\;+\;
\underbrace{O\big({\color{dropred}M^{2}B^{2}}\big)}_{\text{world model bias}}
\;+\;
\underbrace{O\big({\color{dropred}\gamma\, V_{WM}}\big)}_{\text{world model variance}} ,
\end{equation*}
where $V_{WM} = (1 + c\,\sigma^{2})\,V_{E}$ carries
the noise, and the variance term is presented as $O({\color{dropred}\sigma^{2}})$ in
Remark~\ref{rem:price}, and the hidden constants depend only on $L$
and $\mu$. We put the full proof in Appendix~\ref{app:proof-wm}.
\end{theorem}
\end{thmbox}

The first term is the standard geometric decay of the noise-free RL setting and
vanishes as training proceeds. The bias term contains neither $T$ nor
$\gamma$, hence no amount of training and no choice of step size
removes it, and the achievable score stays capped by the bias of the
world model. The variance term grows with the noise through $V_{WM}$ and likewise
inflates the final error, so both terms call for correction.

\subsection{The Effect of Our Mitigation}
\label{subsec:theory-fix}

We now analyze \ours{}, whose two corrections from
Section~\ref{subsec:anchor} improve one term of
Theorem~\ref{thm:wm-conv} each. Both corrections rely on the anchor
signal, where real execution grades a small fraction of the groups and
yields score pairs throughout training. With these pairs, the Online
Debiasing of Equation~\eqref{eq:isotonic} keeps shrinking the residual
bias, at a speed captured by a constant $T_0$ given in
Appendix~\ref{app:conditions}. The Inverse-Variance Denoising of
Equation~\eqref{eq:ivw} then fuses the two gradient streams, so every
update carries less noise than either stream alone. In the statement
below, $\widetilde{O}$ hides logarithmic factors.

\begin{thmbox}
\begin{theorem}[Convergence of \ours{}]
\label{thm:ours-conv}
Under Assumption~\ref{ass:reg} and the conditions of
Appendix~\ref{app:conditions}, \ours{} satisfies with high probability
\begin{equation*}
J^\star - \mathbb{E}[J(\theta_T)] \;\le\;
\Big(1-\tfrac{\gamma\mu}{4}\Big)^{T}\Delta_0
\;+\;
\underbrace{\widetilde{O}\Big(\frac{{\color{dropred}M^{2}B^{2}}}{{\color{gaingreen}1 + T/T_0}}\Big)}_{\text{contractive bias}}
\;+\;
\underbrace{O\Big(\frac{{\color{dropred}\gamma\, V_{WM}}}
{{\color{gaingreen}1 + V_{WM}/V_{E}}}\Big)}_{\text{reduced variance}} .
\end{equation*}
We put the full proof in Appendix~\ref{app:proof-ours}.
\end{theorem}
\end{thmbox}

The two bounds share the same leading term and align term by term,
with each remaining term of Theorem~\ref{thm:ours-conv} equal to its
counterpart in Theorem~\ref{thm:wm-conv} divided by a factor larger
than one. The bias term is the same $M^{2}B^{2}$ divided by $1 +
T/T_0$, so the permanent floor contracts as training proceeds and
vanishes in the limit. The variance term is the same $\gamma
V_{WM}$ divided by $1 + V_{WM}/V_{E}$,
which lands below what either stream attains alone, and the factor
grows as the anchor stream becomes comparatively more reliable. Both
terms are therefore strictly smaller once the world model is biased or
noisy at all, and letting $T \to \infty$ removes the bias term
entirely. In this sense \ours{} converges to
the same optimum as training with real execution, while paying for
real execution on a small fraction of the groups.

\section{Experiments}
\label{sec:experiment}

We empirically validate the effectiveness of \ours{} on the two questions in Section~\ref{sec:introduction}, i.e., whether the world model delivers the promised speedup (question~\textbf{(i)}) and whether it costs any final performance (question~\textbf{(ii)}), with the main comparison in Section~\ref{subsec:exp-main} (Table~\ref{tab:main}). 
We then show how our recipe generalizes to other post-training domains in Section~\ref{subsec:exp-vla} (Table~\ref{tab:vla}) and how the two corrections act in Section~\ref{subsec:exp-ablation} (Table~\ref{tab:ablation}).

\subsection{Experimental Setup}
\label{subsec:exp-setup}

\paragraph{Benchmarks.}
We evaluate \ours{} on various domains, including AutoResearch tasks and vision-language-action (VLA) post-training tasks. 
In the AutoResearch evaluation, we use MLE-Dojo~\citep{qiang2025mledojo}, an interactive superset of MLE-Bench~\citep{chan2024mlebench} built on Kaggle machine learning competitions, and follow the common practice~\citep{liu2025mlagent,cai2026acegrpo,zhou2026sandmle} of manually dividing the pool into MLE-Dojo (train) and MLE-Dojo (test) in a category-balanced way, since the original test split of MLE-Bench has several evaluation issues (Appendix~\ref{app:mle-issues}).
We also put the selection rule and the full task lists in
Appendix~\ref{app:mle-issues}.
Beyond the MLE-style tasks, we further evaluate on
DSBench~\citep{jing2024dsbench}, a data science benchmark also built
on Kaggle modeling tasks and fully disjoint from our training set. In
the VLA evaluation, we use LIBERO-Long~\citep{liu2023libero}, a
standard simulated manipulation benchmark, where the policy controls a
robot arm to complete long-horizon tasks from image observations and
language instructions.

\paragraph{Models and training.}
For the AutoResearch tasks, we post-train research agents at two
scales, Qwen3.5-4B and Qwen3.5-9B~\citep{qwen2026qwen35}, with the
RL procedure of Section~\ref{sec:method}. For the
world model, we use the same backbone as the agent at each scale and
prompt it for outcome prediction (Appendix~\ref{app:wm-prompt}).
Sharing one backbone also rules out any external knowledge, so the
gains cannot come from implicitly distilling a stronger model.
Throughout training, only the research agent is updated while the
world model stays unchanged. All the runs use identical A100 GPU
allocations, and we report the total training compute in GPU-hours. For the VLA
tasks, the agent is MiniVLA-1B~\citep{belkhale2024minivla}, which is a compact
vision-language-action model pretrained on LIBERO-90, and the world
model is Robometer~\citep{liang2026robometer}, an off-the-shelf VLM that predicts task success rate from images. We first finetune the agent on the official demonstrations of each task, and then post-train it with the same RL procedure. The remaining details and hyperparameters are listed in Appendix~\ref{app:experiment_details}.

\begin{table}[t]
\centering
\caption{\textbf{Main results (leaderboard percentile in \%, higher is
better).} Both benchmarks hold out tasks never trained on, and
\textbf{GPU-hours} count the total training compute. Scores are
per-task avg@8, averaged within each category, and \textbf{Avg} is the
mean over categories. \textbf{Bold} marks the best value per column,
and green (red) arrows mark the gain (drop) of \ours{} over same-scale
GRPO.}
\label{tab:main}
\setlength{\tabcolsep}{4pt}
\renewcommand{\arraystretch}{1.15}
\resizebox{\textwidth}{!}{%
\begin{tabular}{l c *{4}{>{\centering\arraybackslash}p{1.62cm}} !{\vrule width 0.6pt} *{5}{>{\centering\arraybackslash}p{1.62cm}}}
\toprule
& & \multicolumn{4}{c}{\textbf{MLE-Dojo (test)}} &
\multicolumn{5}{c}{\textbf{DSBench}} \\
\cmidrule(lr){3-6}\cmidrule(lr){7-11}
\textbf{Method} & \textbf{GPU-hours} & Tab & Text & Img & \textbf{Avg} &
Bin-Cls & Multi-Cls & Regress & Other & \textbf{Avg} \\
\midrule
\rowcolor{gray!14}\multicolumn{11}{l}{\textbf{Baseline Models}}\\
Qwen3.5-4B & --- & 13.7 & 8.0 & 0.2 & 7.3 & 21.9 & 11.8 & 25.8 & 8.8 & 17.1 \\
Qwen3.5-9B & --- & 17.4 & 8.1 & 3.4 & 9.6 & 29.1 & 16.0 & 36.3 & 14.1 & 23.9 \\
\rowcolor{gray!14}\multicolumn{11}{l}{\textbf{Large Agents}}\\
Kimi-48B-A3B & --- & 12.3 & 10.1 & 2.0 & 8.1 & 20.3 & 10.0 & 29.0 & 9.8 & 17.3 \\
Nemotron-120B-A12B & --- & 35.1 & 18.8 & \textbf{7.6} & 20.5 & 38.7 & 26.6 & \textbf{40.7} & 20.6 & 31.7 \\
\rowcolor{gray!14}\multicolumn{11}{l}{\textbf{RL in Real Environment}}\\
Qwen3.5-4B-GRPO & 883 & 24.8 & 16.6 & 4.3 & 15.2 & 34.3 & 17.8 & 28.3 & 22.4 & 25.7 \\
Qwen3.5-9B-GRPO & 1174 & 33.3 & 17.1 & 5.9 & 18.8 & 40.9 & 26.8 & 40.4 & 16.7 & 31.2 \\
\rowcolor{gray!14}\multicolumn{11}{l}{\textbf{RL with Pure World Model}}\\
Qwen3.5-4B-WM & 269 & 21.2 & 14.6 & 3.0 & 12.9 & 27.1 & 16.9 & 32.5 & 16.0 & 23.1 \\
Qwen3.5-9B-WM & 330 & 28.0 & 15.9 & 4.5 & 16.1 & 35.4 & 21.8 & 37.9 & 16.9 & 28.0 \\
\rowcolor{gray!14}\multicolumn{11}{l}{\textbf{\ours{} (ours)}}\\
Qwen3.5-4B-Ours & 286 & 28.0\gain{3.2} & 16.1\drop{0.5} & 5.2\gain{0.9} & 16.4\gain{1.2} & 35.8\gain{1.5} & 22.3\gain{4.5} & 32.3\gain{4.0} & \textbf{24.7}\gain{2.3} & 28.8\gain{3.1} \\
Qwen3.5-9B-Ours & 349 & \textbf{38.0}\gain{4.7} & \textbf{19.4}\gain{2.3} & 7.4\gain{1.5} & \textbf{21.6}\gain{2.8} & \textbf{45.4}\gain{4.5} & \textbf{26.9}\gain{0.1} & 39.2\drop{1.2} & 19.8\gain{3.1} & \textbf{32.8}\gain{1.6} \\
\bottomrule
\end{tabular}%
}
\end{table}

\paragraph{Metrics.}
For the AutoResearch tasks, we adopt the official metric of MLE-Dojo ~\citep{qiang2025mledojo,cai2026acegrpo}, which scores each submission by its percentile on the corresponding real competition leaderboard. The percentile is naturally normalized (within $0-1$), so scores are comparable across tasks. For the VLA tasks, we score each
rollout by whether it succeeds. In both domains, every result we report is an average over eight attempts (avg@8).

\paragraph{Baselines.}
We compare \ours{} with four types of baselines. The untrained bases
Qwen3.5-4B and Qwen3.5-9B give the starting level of the backbone
before any post-training. The large open-weight agents Kimi-48B-A3B
and Nemotron-120B-A12B serve as strong off-the-shelf references and
test whether model scale alone can replace post-training. GRPO in the
real environment is the standard full-cost training that \ours{} aims
to match, and is our primary comparison. The pure world model baseline
trains on predicted rewards alone and shows what the raw world model
delivers without any correction. All trained
models share the scaffold, the hyperparameters, and the data split.

\subsection{Main Results on AutoResearch Tasks}
\label{subsec:exp-main}

Table~\ref{tab:main} compares \ours{} with the four types of baselines
of Section~\ref{subsec:exp-setup}, at both scales and on both
benchmarks, and examines whether \ours{} delivers the promised speedup
and whether the cheap signal costs any final performance. The answer
is positive on both counts. First, \ours{} cuts the training compute
of real-execution GRPO by $3.1\times$ and $3.4\times$ and still scores
higher on every benchmark, with gains of up to 3.1 points, so the saving arrives at no cost in final performance. 
Second, the post-trained agents beat far larger off-the-shelf agents, as our 4B model surpasses the 48B agent and our 9B model surpasses the 120B agent on both averages.

\subsection{Generalization Ability on Embodied VLAs}
\label{subsec:exp-vla}

We next test whether \ours{} transfers to other domains, and we pick
vision-language-action (VLA) post-training, an active line of embodied
learning~\citep{kim2024openvla,black2025pi05,bai2026latent}. The two
calibration mechanisms are not specific to AutoResearch, so we apply
the same recipe to the VLA setup of
Section~\ref{subsec:exp-setup}. Here Robometer predicts a progress
value for eight frames sampled from each rollout as the dense reward,
and the single sparse success that the environment returns at the end
of a rollout serves as the anchor.
The resulting success rates in Table~\ref{tab:vla} repeat the pattern
of Table~\ref{tab:main}. Either signal alone barely moves the
policy, as RL on the sparse outcome adds 0.9 points over the SFT
baseline and RL on the raw world model signal adds 1.8. \ours{}
combines the two and lifts the overall success rate by 3.8 points,
with the largest margin on unseen initial states. The recipe therefore
carries over across domains, agent architectures, and reward types.

\begin{table}[t]
\centering
\caption{\textbf{VLA post-training results on LIBERO-Long (success
rate in \%, higher is better).} We report in-domain performance,
out-of-distribution generalization, and the overall average, where
Overall averages over every initial state, seen and unseen. Best@8 counts at least one
success of the eight and All@8 counts all eight. All RL rows share the
SFT initialization.
\textbf{Bold} marks the best value per column, and green (red) arrows
mark the gain (drop) of \ours{} over MiniVLA-1B-GRPO.}
\label{tab:vla}
\setlength{\tabcolsep}{4pt}
\renewcommand{\arraystretch}{1.15}
\resizebox{\textwidth}{!}{%
\begin{tabular}{l *{3}{>{\centering\arraybackslash}p{2.45cm}} !{\vrule width 0.6pt} *{3}{>{\centering\arraybackslash}p{2.45cm}} !{\vrule width 0.6pt} >{\centering\arraybackslash}p{2.45cm}}
\toprule
& \multicolumn{3}{c}{\textbf{In-Domain}} &
\multicolumn{3}{c}{\textbf{Out-of-Distribution}} & \\
\cmidrule(lr){2-4}\cmidrule(lr){5-7}
\textbf{Method} & Avg & Best@8 & All@8 & Avg & Best@8 & All@8 & \textbf{Overall} \\
\midrule
\rowcolor{gray!14}\multicolumn{8}{l}{\textbf{Baseline Models}}\\
MiniVLA-1B & 5.6 & 5.6 & 5.6 & 2.9 & 2.9 & 2.9 & 3.8 \\
MiniVLA-1B-SFT & 37.3 & 41.3 & 33.1 & 37.5 & 44.7 & 32.1 & 37.4 \\
\rowcolor{gray!14}\multicolumn{8}{l}{\textbf{RL in Real Environment}}\\
MiniVLA-1B-GRPO & 39.3 & \textbf{48.8} & 34.4 & 37.8 & 45.3 & 31.2 & 38.3 \\
\rowcolor{gray!14}\multicolumn{8}{l}{\textbf{RL with Pure World Model}}\\
MiniVLA-1B-WM & 39.1 & 45.6 & 34.4 & 39.3 & 45.6 & 34.1 & 39.2 \\
\rowcolor{gray!14}\multicolumn{8}{l}{\textbf{\ours{} (ours)}}\\
MiniVLA-1B-Ours & \textbf{41.2}\gain{1.9} & 47.5\drop{1.3} & \textbf{37.5}\gain{3.1} & \textbf{41.2}\gain{3.4} & \textbf{48.8}\gain{3.5} & \textbf{37.1}\gain{5.9} & \textbf{41.2}\gain{2.9} \\
\bottomrule
\end{tabular}%
}
\end{table}

\subsection{Ablation Study}
\label{subsec:exp-ablation}

\begin{wraptable}[11]{r}{0.62\textwidth}
\vspace{-14pt}
\centering
\caption{\textbf{Ablation on the two corrections (leaderboard
percentile in \%, Avg over categories).} OD is Online Debiasing, IVD
is Inverse-Variance Denoising, a filled circle marks the correction
as active, and the last row recovers \ours{}. Green arrows give its
gain over the uncorrected first row.}
\label{tab:ablation}
\setlength{\tabcolsep}{9pt}
\renewcommand{\arraystretch}{1.05}
\footnotesize
\begin{tabular}{c c cc cc}
\toprule
& & \multicolumn{2}{c}{\textbf{4B Agent}} &
\multicolumn{2}{c}{\textbf{9B Agent}} \\
\cmidrule(lr){3-4}\cmidrule(lr){5-6}
\textbf{OD} & \textbf{IVD} & MLE & DS & MLE & DS \\
\midrule
$\circ$ & $\circ$ & 13.5 & 25.3 & 16.8 & 29.5 \\
$\circ$ & $\bullet$ & 14.9 & 26.2 & 18.0 & 31.2 \\
$\bullet$ & $\circ$ & 15.7 & 28.1 & 19.4 & 31.7 \\
$\bullet$ & $\bullet$ & \textbf{16.4}\gain{2.9} & \textbf{28.8}\gain{3.5} & \textbf{21.6}\gain{4.8} & \textbf{32.8}\gain{3.3} \\
\bottomrule
\end{tabular}
\vspace{-6pt}
\end{wraptable}
We then ablate the two corrections at both scales
(Table~\ref{tab:ablation}). All runs consume the same two reward
streams at the same ratio as \ours{} and differ only in which
correction is active. In particular, the first row is the direct
mixture that feeds both signals to GRPO with no correction at all.
From this baseline we add Inverse-Variance Denoising alone, Online
Debiasing alone, and finally both, which recovers \ours{}.
The plain mixture is not enough, as its scores stay
below real-execution GRPO (Table~\ref{tab:main}) on every column.
Each correction then helps on its own. Inverse-Variance Denoising
alone adds 0.9 to 1.7 points, Online Debiasing alone adds 2.2 to 2.8
points, and the larger share of the debiasing gain matches
Theorem~\ref{thm:wm-conv}, where the bias enters the bound at full size
while the noise enters damped by the step size. Activating both lifts every
column by 2.9 to 4.8 points, more than either correction alone, so
the two mechanisms are complementary rather than redundant, as each
removes the term the other leaves behind. 

\section{Conclusion}
\label{sec:conclusion}

This work scales RL for AutoResearch agents by replacing the expensive
environment execution with a world model and correcting its bias and
noise through a small anchored stream of real execution. The two
corrections turn the permanent error floor of world model training
into a contracting term and reduce the variance below either reward
stream alone, and they cut the training compute by three to four times
while matching or exceeding full real-execution RL at two scales. The
transfer to VLA post-training further suggests a general path for
scaling RL wherever execution, not generation, is the bottleneck.

\clearpage

\bibliographystyle{unsrtnat}
\bibliography{reference}

\begin{thebibliography}{66}
\providecommand{\natexlab}[1]{#1}
\providecommand{\url}[1]{\texttt{#1}}
\expandafter\ifx\csname urlstyle\endcsname\relax
  \providecommand{\doi}[1]{doi: #1}\else
  \providecommand{\doi}{doi: \begingroup \urlstyle{rm}\Url}\fi

\bibitem[Lu et~al.(2024)Lu, Lu, Lange, Foerster, Clune, and Ha]{lu2024ai}
Chris Lu, Cong Lu, Robert~Tjarko Lange, Jakob Foerster, Jeff Clune, and David
  Ha.
\newblock The {AI} scientist: Towards fully automated open-ended scientific
  discovery.
\newblock \emph{arXiv preprint arXiv:2408.06292}, 2024.

\bibitem[Gottweis et~al.(2025)Gottweis, Weng, Daryin, Tu, Palepu, Sirkovic,
  Myaskovsky, Weissenberger, Rong, Tanno, et~al.]{gottweis2025aicoscientist}
Juraj Gottweis, Wei-Hung Weng, Alexander Daryin, Tao Tu, Anil Palepu, Petar
  Sirkovic, Artiom Myaskovsky, Felix Weissenberger, Keran Rong, Ryutaro Tanno,
  et~al.
\newblock Towards an {AI} co-scientist.
\newblock \emph{arXiv preprint arXiv:2502.18864}, 2025.

\bibitem[Schmidgall et~al.(2025)Schmidgall, Su, Wang, Sun, Wu, Yu, Liu, Liu,
  and Barsoum]{schmidgall2025agentlab}
Samuel Schmidgall, Yusheng Su, Ze~Wang, Ximeng Sun, Jialian Wu, Xiaodong Yu,
  Jiang Liu, Zicheng Liu, and Emad Barsoum.
\newblock Agent laboratory: Using {LLM} agents as research assistants.
\newblock \emph{arXiv preprint arXiv:2501.04227}, 2025.

\bibitem[Yao et~al.(2023{\natexlab{a}})Yao, Zhao, Yu, Du, Shafran, Narasimhan,
  and Cao]{yao2023react}
Shunyu Yao, Jeffrey Zhao, Dian Yu, Nan Du, Izhak Shafran, Karthik Narasimhan,
  and Yuan Cao.
\newblock {ReAct}: Synergizing reasoning and acting in language models.
\newblock In \emph{International Conference on Learning Representations},
  2023{\natexlab{a}}.

\bibitem[Yao et~al.(2023{\natexlab{b}})Yao, Yu, Zhao, Shafran, Griffiths, Cao,
  and Narasimhan]{yao2023tot}
Shunyu Yao, Dian Yu, Jeffrey Zhao, Izhak Shafran, Tom Griffiths, Yuan Cao, and
  Karthik Narasimhan.
\newblock Tree of thoughts: Deliberate problem solving with large language
  models.
\newblock In \emph{Advances in Neural Information Processing Systems},
  2023{\natexlab{b}}.

\bibitem[Shinn et~al.(2023)Shinn, Cassano, Gopinath, Narasimhan, and
  Yao]{shinn2023reflexion}
Noah Shinn, Federico Cassano, Ashwin Gopinath, Karthik Narasimhan, and Shunyu
  Yao.
\newblock Reflexion: Language agents with verbal reinforcement learning.
\newblock In \emph{Advances in Neural Information Processing Systems}, 2023.

\bibitem[Boiko et~al.(2023)Boiko, MacKnight, Kline, and
  Gomes]{boiko2023autonomous}
Daniil~A. Boiko, Robert MacKnight, Ben Kline, and Gabe Gomes.
\newblock Autonomous chemical research with large language models.
\newblock \emph{Nature}, 624\penalty0 (7992):\penalty0 570--578, 2023.

\bibitem[Bran et~al.(2024)Bran, Cox, Schilter, Baldassari, White, and
  Schwaller]{bran2024chemcrow}
Andres~M Bran, Sam Cox, Oliver Schilter, Carlo Baldassari, Andrew~D White, and
  Philippe Schwaller.
\newblock Augmenting large language models with chemistry tools.
\newblock \emph{Nature Machine Intelligence}, 6\penalty0 (5):\penalty0
  525--535, 2024.

\bibitem[Szymanski et~al.(2023)Szymanski, Rendy, Fei, Kumar, He, Milsted,
  McDermott, Gallant, Cubuk, Merchant, et~al.]{szymanski2023alab}
Nathan~J Szymanski, Bernardus Rendy, Yuxing Fei, Rishi~E Kumar, Tanjin He,
  David Milsted, Matthew~J McDermott, Max Gallant, Ekin~Dogus Cubuk, Amil
  Merchant, et~al.
\newblock An autonomous laboratory for the accelerated synthesis of novel
  inorganic materials.
\newblock \emph{Nature}, 624:\penalty0 86--91, 2023.

\bibitem[Jiang et~al.(2025)Jiang, Schmidt, Srikanth, Xu, Kaplan, Jacenko, and
  Wu]{jiang2025aide}
Zhengyao Jiang, Dominik Schmidt, Dhruv Srikanth, Dixing Xu, Ian Kaplan, Deniss
  Jacenko, and Yuxiang Wu.
\newblock {AIDE}: {AI}-driven exploration in the space of code.
\newblock \emph{arXiv preprint arXiv:2502.13138}, 2025.

\bibitem[Guo et~al.(2024)Guo, Deng, Wen, Chen, Chang, and Wang]{guo2024dsagent}
Siyuan Guo, Cheng Deng, Ying Wen, Hechang Chen, Yi~Chang, and Jun Wang.
\newblock {DS-Agent}: Automated data science by empowering large language
  models with case-based reasoning.
\newblock In \emph{International Conference on Machine Learning}, 2024.

\bibitem[Huang et~al.(2024)Huang, Vora, Liang, and
  Leskovec]{huang2024mlagentbench}
Qian Huang, Jian Vora, Percy Liang, and Jure Leskovec.
\newblock {MLAgentBench}: Evaluating language agents on machine learning
  experimentation.
\newblock In \emph{International Conference on Machine Learning}, 2024.

\bibitem[Schulman et~al.(2017)Schulman, Wolski, Dhariwal, Radford, and
  Klimov]{schulman2017ppo}
John Schulman, Filip Wolski, Prafulla Dhariwal, Alec Radford, and Oleg Klimov.
\newblock Proximal policy optimization algorithms.
\newblock \emph{arXiv preprint arXiv:1707.06347}, 2017.

\bibitem[Shao et~al.(2024)Shao, Wang, Zhu, Xu, Song, Bi, Zhang, Zhang,
  et~al.]{shao2024deepseekmath}
Zhihong Shao, Peiyi Wang, Qihao Zhu, Runxin Xu, Junxiao Song, Xiao Bi, Haowei
  Zhang, Mingchuan Zhang, et~al.
\newblock {DeepSeekMath}: Pushing the limits of mathematical reasoning in open
  language models.
\newblock \emph{arXiv preprint arXiv:2402.03300}, 2024.

\bibitem[Lambert et~al.(2024)Lambert, Morrison, Pyatkin, Huang, Ivison,
  Brahman, Miranda, Liu, Dziri, Lyu, et~al.]{lambert2024tulu3}
Nathan Lambert, Jacob Morrison, Valentina Pyatkin, Shengyi Huang, Hamish
  Ivison, Faeze Brahman, Lester James~V Miranda, Alisa Liu, Nouha Dziri, Shane
  Lyu, et~al.
\newblock Tulu 3: Pushing frontiers in open language model post-training.
\newblock \emph{arXiv preprint arXiv:2411.15124}, 2024.

\bibitem[Guo et~al.(2025)Guo, Yang, Zhang, Song, Zhang, Xu, Zhu, Ma, Wang, Bi,
  et~al.]{guo2025deepseekr1}
Daya Guo, Dejian Yang, Haowei Zhang, Junxiao Song, Ruoyu Zhang, Runxin Xu,
  Qihao Zhu, Shirong Ma, Peiyi Wang, Xiao Bi, et~al.
\newblock {DeepSeek-R1}: Incentivizing reasoning capability in {LLMs} via
  reinforcement learning.
\newblock \emph{arXiv preprint arXiv:2501.12948}, 2025.

\bibitem[Kwon et~al.(2023)Kwon, Li, Zhuang, Sheng, Zheng, Yu, Gonzalez, Zhang,
  and Stoica]{kwon2023vllm}
Woosuk Kwon, Zhuohan Li, Siyuan Zhuang, Ying Sheng, Lianmin Zheng, Cody~Hao Yu,
  Joseph~E. Gonzalez, Hao Zhang, and Ion Stoica.
\newblock Efficient memory management for large language model serving with
  pagedattention.
\newblock In \emph{Proceedings of the 29th Symposium on Operating Systems
  Principles}, 2023.

\bibitem[Zheng et~al.(2024)Zheng, Yin, Xie, Sun, Huang, Yu, Cao, Kozyrakis,
  Stoica, Gonzalez, Barrett, and Sheng]{zheng2024sglang}
Lianmin Zheng, Liangsheng Yin, Zhiqiang Xie, Chuyue Sun, Jeff Huang, Cody~Hao
  Yu, Shiyi Cao, Christos Kozyrakis, Ion Stoica, Joseph~E. Gonzalez, Clark
  Barrett, and Ying Sheng.
\newblock {SGLang}: Efficient execution of structured language model programs.
\newblock In \emph{Advances in Neural Information Processing Systems}, 2024.

\bibitem[Qiang et~al.(2025)Qiang, Zhuang, Li, V~K, et~al.]{qiang2025mledojo}
Rushi Qiang, Yuchen Zhuang, Yinghao Li, Dingu~Sagar V~K, et~al.
\newblock {MLE-Dojo}: Interactive environments for empowering {LLM} agents in
  machine learning engineering.
\newblock \emph{arXiv preprint arXiv:2505.07782}, 2025.

\bibitem[Sun et~al.(2026)Sun, Song, Huang, Jiang, Le, Lv, Chen, Hu, Luo, Zhao,
  Song, Xu, Zhang, and Wen]{sun2026sweworld}
Shuang Sun, Huatong Song, Lisheng Huang, Jinhao Jiang, Ran Le, Zhihao Lv,
  Zongchao Chen, Yiwen Hu, Wenyang Luo, Wayne~Xin Zhao, Yang Song, Hongteng Xu,
  Tao Zhang, and Ji-Rong Wen.
\newblock {SWE-World}: Building software engineering agents in docker-free
  environments.
\newblock \emph{arXiv preprint arXiv:2602.03419}, 2026.

\bibitem[Ha and Schmidhuber(2018)]{ha2018world}
David Ha and J{\"u}rgen Schmidhuber.
\newblock World models.
\newblock \emph{arXiv preprint arXiv:1803.10122}, 2018.

\bibitem[Hafner et~al.(2025)Hafner, Pasukonis, Ba, and
  Lillicrap]{hafner2025mastering}
Danijar Hafner, Jurgis Pasukonis, Jimmy Ba, and Timothy Lillicrap.
\newblock Mastering diverse control tasks through world models.
\newblock \emph{Nature}, 640\penalty0 (8059):\penalty0 647--653, 2025.

\bibitem[Novikov et~al.(2025)Novikov, V{\~u}, Eisenberger,
  et~al.]{novikov2025alphaevolve}
Alexander Novikov, Ng{\^a}n V{\~u}, Marvin Eisenberger, et~al.
\newblock Alphaevolve: A coding agent for scientific and algorithmic discovery.
\newblock \emph{arXiv preprint arXiv:2506.13131}, 2025.

\bibitem[Yamada et~al.(2025)Yamada, Lange, Lu, Hu, Lu, Foerster, Clune, and
  Ha]{yamada2025aiscientistv2}
Yutaro Yamada, Robert~Tjarko Lange, Cong Lu, Shengran Hu, Chris Lu, Jakob
  Foerster, Jeff Clune, and David Ha.
\newblock The {AI} scientist-v2: Workshop-level automated scientific discovery
  via agentic tree search.
\newblock \emph{arXiv preprint arXiv:2504.08066}, 2025.

\bibitem[Kulikov et~al.(2026)Kulikov, Whitehouse, Wu, Nie, Saha, Helenowski,
  Yuan, Golovneva, Lanchantin, Bachrach, Foerster, Li, Fang, Sukhbaatar, and
  Weston]{kulikov2026autodata}
Ilia Kulikov, Chenxi Whitehouse, Tianhao Wu, Yixin Nie, Swarnadeep Saha, Eryk
  Helenowski, Weizhe Yuan, Olga Golovneva, Jack Lanchantin, Yoram Bachrach,
  Jakob Foerster, Xian Li, Han Fang, Sainbayar Sukhbaatar, and Jason Weston.
\newblock Autodata: An agentic data scientist to create high quality synthetic
  data.
\newblock \emph{arXiv preprint arXiv:2606.25996}, 2026.

\bibitem[Bertran et~al.(2026)Bertran, Fogliato, and Wu]{bertran2026many}
Martin Bertran, Riccardo Fogliato, and Zhiwei~Steven Wu.
\newblock Many ai analysts, one dataset: Navigating the agentic data science
  multiverse.
\newblock \emph{Proceedings of the National Academy of Sciences}, 123\penalty0
  (29):\penalty0 e2606495123, 2026.

\bibitem[Yang et~al.(2026)Yang, Jiang, Fu, Luo, Ren, Wang, Zhao, Liu,
  et~al.]{yang2026frontis}
Junlin Yang, Che Jiang, Yu~Fu, Tianwei Luo, Can Ren, Weizhi Wang, Kaikai Zhao,
  Hongyi Liu, et~al.
\newblock Frontis-{MA1}: Training an {AI4AI} model towards recursive
  self-improvement in machine learning engineering.
\newblock \emph{arXiv preprint arXiv:2607.28568}, 2026.

\bibitem[{Recursive}(2026)]{recursive2026firststeps}
{Recursive}.
\newblock First steps toward automated {AI} research.
\newblock
  \url{https://www.recursive.com/articles/first-steps-toward-automated-ai-research},
  2026.

\bibitem[Chan et~al.(2024)Chan, Chowdhury, Jaffe, Aung, Sherburn, Mays,
  Starace, Liu, et~al.]{chan2024mlebench}
Jun~Shern Chan, Neil Chowdhury, Oliver Jaffe, James Aung, Dane Sherburn, Evan
  Mays, Giulio Starace, Kevin Liu, et~al.
\newblock {MLE-bench}: Evaluating machine learning agents on machine learning
  engineering.
\newblock \emph{arXiv preprint arXiv:2410.07095}, 2024.

\bibitem[Jing et~al.(2024)Jing, Huang, Wang, Yao, Yu, Ma, Zhang, Du,
  et~al.]{jing2024dsbench}
Liqiang Jing, Zhehui Huang, Xiaoyang Wang, Wenlin Yao, Wenhao Yu, Kaixin Ma,
  Hongming Zhang, Xinya Du, et~al.
\newblock {DSBench}: How far are data science agents from becoming data science
  experts?
\newblock \emph{arXiv preprint arXiv:2409.07703}, 2024.

\bibitem[Nathani et~al.(2025)Nathani, Madaan, Roberts, Bashlykov,
  et~al.]{nathani2025mlgym}
Deepak Nathani, Lovish Madaan, Nicholas Roberts, Nikolay Bashlykov, et~al.
\newblock {MLGym}: A new framework and benchmark for advancing {AI} research
  agents.
\newblock \emph{arXiv preprint arXiv:2502.14499}, 2025.

\bibitem[Jimenez et~al.(2024)Jimenez, Yang, Wettig, Yao, Pei, Press, and
  Narasimhan]{jimenez2024swebench}
Carlos~E Jimenez, John Yang, Alexander Wettig, Shunyu Yao, Kexin Pei, Ofir
  Press, and Karthik Narasimhan.
\newblock {SWE-bench}: Can language models resolve real-world {GitHub} issues?
\newblock In \emph{International Conference on Learning Representations}, 2024.

\bibitem[Yang et~al.(2024)Yang, Jimenez, Wettig, Lieret, Yao, Narasimhan, and
  Press]{yang2024sweagent}
John Yang, Carlos~E Jimenez, Alexander Wettig, Kilian Lieret, Shunyu Yao,
  Karthik Narasimhan, and Ofir Press.
\newblock {SWE-agent}: Agent-computer interfaces enable automated software
  engineering.
\newblock In \emph{Advances in Neural Information Processing Systems}, 2024.

\bibitem[Wang et~al.(2025)Wang, Li, Song, Xu, Tang, Zhuge, Pan, Song, Li,
  Singh, et~al.]{wang2025openhands}
Xingyao Wang, Boxuan Li, Yufan Song, Frank~F Xu, Xiangru Tang, Mingchen Zhuge,
  Jiayi Pan, Yueqi Song, Bowen Li, Jaskirat Singh, et~al.
\newblock {OpenHands}: An open platform for {AI} software developers as
  generalist agents.
\newblock In \emph{International Conference on Learning Representations}, 2025.

\bibitem[Pan et~al.(2025)Pan, Wang, Neubig, Jaitly, Ji, Suhr, and
  Zhang]{pan2025swegym}
Jiayi Pan, Xingyao Wang, Graham Neubig, Navdeep Jaitly, Heng Ji, Alane Suhr,
  and Yizhe Zhang.
\newblock Training software engineering agents and verifiers with {SWE-Gym}.
\newblock In \emph{International Conference on Machine Learning}, 2025.

\bibitem[Sheng et~al.(2024)Sheng, Zhang, Ye, Wu, et~al.]{sheng2024hybridflow}
Guangming Sheng, Chi Zhang, Zilingfeng Ye, Xibin Wu, et~al.
\newblock {HybridFlow}: A flexible and efficient {RLHF} framework.
\newblock \emph{arXiv preprint arXiv:2409.19256}, 2024.

\bibitem[Xie et~al.(2026)Xie, Lin, Wang, Ning, Yao, Xue, Zhang, Li, Zhang, Wu,
  Chen, Gou, Han, Wang, Lee, Wei, Wang, Su, and Sun]{xie2026quest}
Jian Xie, Tianhe Lin, Zilu Wang, Yuting Ning, Yuekun Yao, Tianci Xue, Zhehao
  Zhang, Zhongyang Li, Kai Zhang, Yufan Wu, Shijie Chen, Boyu Gou, Mingzhe Han,
  Yifei Wang, Vint Lee, Xinpeng Wei, Xiangjun Wang, Yu~Su, and Huan Sun.
\newblock {QUEST}: Training frontier deep research agents with fully synthetic
  tasks.
\newblock \emph{arXiv preprint arXiv:2605.24218}, 2026.

\bibitem[Yu et~al.(2025)Yu, Zhang, Zhu, Yuan, Zuo, Yue, Fan, Liu, Liu, Liu,
  et~al.]{yu2025dapo}
Qiying Yu, Zheng Zhang, Ruofei Zhu, Yufeng Yuan, Xiaochen Zuo, Yu~Yue, Tiantian
  Fan, Gaohong Liu, Lingjun Liu, Xin Liu, et~al.
\newblock {DAPO}: An open-source {LLM} reinforcement learning system at scale.
\newblock \emph{arXiv preprint arXiv:2503.14476}, 2025.

\bibitem[Qu et~al.(2026)Qu, Setlur, Smith, Salakhutdinov, and
  Kumar]{qu2026pope}
Yuxiao Qu, Amrith Setlur, Virginia Smith, Ruslan Salakhutdinov, and Aviral
  Kumar.
\newblock {POPE}: Learning to reason on hard problems via privileged on-policy
  exploration.
\newblock \emph{arXiv preprint arXiv:2601.18779}, 2026.

\bibitem[Ouyang et~al.(2022)Ouyang, Wu, Jiang, Almeida, Wainwright, Mishkin,
  Zhang, Agarwal, Slama, Ray, Schulman, Hilton, Kelton, Miller, Simens, Askell,
  Welinder, Christiano, Leike, and Lowe]{ouyang2022training}
Long Ouyang, Jeffrey Wu, Xu~Jiang, Diogo Almeida, Carroll~L. Wainwright, Pamela
  Mishkin, Chong Zhang, Sandhini Agarwal, Katarina Slama, Alex Ray, John
  Schulman, Jacob Hilton, Fraser Kelton, Luke Miller, Maddie Simens, Amanda
  Askell, Peter Welinder, Paul~F. Christiano, Jan Leike, and Ryan Lowe.
\newblock Training language models to follow instructions with human feedback.
\newblock In \emph{Advances in Neural Information Processing Systems},
  volume~35, pages 27730--27744, 2022.

\bibitem[Zheng et~al.(2023)Zheng, Chiang, Sheng, Zhuang, Wu, Zhuang, Lin, Li,
  Li, Xing, Zhang, Gonzalez, and Stoica]{zheng2023judging}
Lianmin Zheng, Wei-Lin Chiang, Ying Sheng, Siyuan Zhuang, Zhanghao Wu, Yonghao
  Zhuang, Zi~Lin, Zhuohan Li, Dacheng Li, Eric~P. Xing, Hao Zhang, Joseph~E.
  Gonzalez, and Ion Stoica.
\newblock Judging {LLM}-as-a-judge with {MT}-bench and chatbot arena.
\newblock In \emph{Advances in Neural Information Processing Systems 36:
  Datasets and Benchmarks Track}, volume~36, pages 46595--46623, 2023.

\bibitem[Bruce et~al.(2024)Bruce, Dennis, Edwards, Parker-Holder, Shi, Hughes,
  Lai, Mavalankar, Steigerwald, Apps, et~al.]{bruce2024genie}
Jake Bruce, Michael~D Dennis, Ashley Edwards, Jack Parker-Holder, Yuge Shi,
  Edward Hughes, Matthew Lai, Aditi Mavalankar, Richie Steigerwald, Chris Apps,
  et~al.
\newblock Genie: Generative interactive environments.
\newblock In \emph{International Conference on Machine Learning}, 2024.

\bibitem[{NVIDIA}(2025)]{nvidia2025cosmos}
{NVIDIA}.
\newblock Cosmos world foundation model platform for physical {AI}.
\newblock \emph{arXiv preprint arXiv:2501.03575}, 2025.

\bibitem[Dainese et~al.(2024)Dainese, Merler, Alakuijala, and
  Marttinen]{dainese2024codewm}
Nicola Dainese, Matteo Merler, Minttu Alakuijala, and Pekka Marttinen.
\newblock Generating code world models with large language models guided by
  monte carlo tree search.
\newblock In \emph{Advances in Neural Information Processing Systems}, 2024.

\bibitem[Tang et~al.(2024)Tang, Key, and Ellis]{tang2024worldcoder}
Hao Tang, Darren Key, and Kevin Ellis.
\newblock Worldcoder, a model-based {LLM} agent: Building world models by
  writing code and interacting with the environment.
\newblock In \emph{Advances in Neural Information Processing Systems}, 2024.

\bibitem[Gao et~al.(2023)Gao, Schulman, and Hilton]{gao2023scaling}
Leo Gao, John Schulman, and Jacob Hilton.
\newblock Scaling laws for reward model overoptimization.
\newblock In \emph{Proceedings of the 40th International Conference on Machine
  Learning}, volume 202, pages 10835--10866, 2023.

\bibitem[Coste et~al.(2024)Coste, Anwar, Kirk, and
  Krueger]{coste2024rewardensembles}
Thomas Coste, Usman Anwar, Robert Kirk, and David Krueger.
\newblock Reward model ensembles help mitigate overoptimization.
\newblock In \emph{International Conference on Learning Representations}, 2024.

\bibitem[Ajalloeian and Stich(2020)]{ajalloeian2020convergence}
Ahmad Ajalloeian and Sebastian~U. Stich.
\newblock On the convergence of {SGD} with biased gradients.
\newblock \emph{arXiv preprint arXiv:2008.00051}, 2020.

\bibitem[Zadrozny and Elkan(2002)]{zadrozny2002calibration}
Bianca Zadrozny and Charles Elkan.
\newblock Transforming classifier scores into accurate multiclass probability
  estimates.
\newblock In \emph{Proceedings of the ACM SIGKDD International Conference on
  Knowledge Discovery and Data Mining}, 2002.

\bibitem[Guo et~al.(2017)Guo, Pleiss, Sun, and Weinberger]{guo2017calibration}
Chuan Guo, Geoff Pleiss, Yu~Sun, and Kilian~Q. Weinberger.
\newblock On calibration of modern neural networks.
\newblock In \emph{Proceedings of the 34th International Conference on Machine
  Learning}, volume~70, pages 1321--1330, 2017.

\bibitem[Nachum et~al.(2019)Nachum, Chow, Dai, and Li]{nachum2019dualdice}
Ofir Nachum, Yinlam Chow, Bo~Dai, and Lihong Li.
\newblock {DualDICE}: Behavior-agnostic estimation of discounted stationary
  distribution corrections.
\newblock In \emph{Advances in Neural Information Processing Systems}, 2019.

\bibitem[Li et~al.(2022)Li, Cheng, Liao, Wang, Wang, and Bai]{li2021offpolicy}
Jiachen Li, Shuo Cheng, Zhenyu Liao, Huayan Wang, William~Yang Wang, and Qinxun
  Bai.
\newblock Off-policy reinforcement learning with optimistic exploration and
  distribution correction.
\newblock In \emph{Deep Reinforcement Learning Workshop, NeurIPS}, 2022.

\bibitem[Liu et~al.(2025)Liu, Chai, Zhu, Tang, Ye, Zhang, Bai, and
  Chen]{liu2025mlagent}
Zexi Liu, Jingyi Chai, Xinyu Zhu, Shuo Tang, Rui Ye, Bo~Zhang, Lei Bai, and
  Siheng Chen.
\newblock {ML-Agent}: Reinforcing {LLM} agents for autonomous machine learning
  engineering.
\newblock \emph{arXiv preprint arXiv:2505.23723}, 2025.

\bibitem[Cai et~al.(2026)Cai, Liu, Zhu, Wang, Wang, and Chen]{cai2026acegrpo}
Yuzhu Cai, Zexi Liu, Xinyu Zhu, Cheng Wang, Yanfeng Wang, and Siheng Chen.
\newblock {AceGRPO}: Adaptive curriculum enhanced group relative policy
  optimization for autonomous machine learning engineering.
\newblock \emph{arXiv preprint arXiv:2602.07906}, 2026.

\bibitem[Zhou et~al.(2026)Zhou, Zhang, Wu, Liu, Fan, Zhao, and
  Yan]{zhou2026sandmle}
Yuhang Zhou, Lizhu Zhang, Yifan Wu, Jiayi Liu, Xiangjun Fan, Zhuokai Zhao, and
  Hong Yan.
\newblock Synthetic sandbox for training machine learning engineering agents.
\newblock \emph{arXiv preprint arXiv:2604.04872}, 2026.

\bibitem[Liu et~al.(2023)Liu, Zhu, Gao, Feng, Liu, Zhu, and
  Stone]{liu2023libero}
Bo~Liu, Yifeng Zhu, Chongkai Gao, Yihao Feng, Qiang Liu, Yuke Zhu, and Peter
  Stone.
\newblock {LIBERO}: Benchmarking knowledge transfer for lifelong robot
  learning.
\newblock In \emph{Advances in Neural Information Processing Systems}, 2023.

\bibitem[{Qwen Team}(2026)]{qwen2026qwen35}
{Qwen Team}.
\newblock Qwen3.5.
\newblock \url{https://qwen.ai/blog?id=qwen3.5}, 2026.

\bibitem[Belkhale and Sadigh(2024)]{belkhale2024minivla}
Suneel Belkhale and Dorsa Sadigh.
\newblock {MiniVLA}: A better {VLA} with a smaller footprint.
\newblock \url{https://github.com/Stanford-ILIAD/openvla-mini}, 2024.

\bibitem[Liang et~al.(2026)Liang, Korkmaz, Zhang, Hwang, Anwar, Kaushik, Shah,
  Huang, Zettlemoyer, Fox, Xiang, Li, Bobu, Gupta, Tu, Biyik, and
  Zhang]{liang2026robometer}
Anthony Liang, Yigit Korkmaz, Jiahui Zhang, Minyoung Hwang, Abrar Anwar,
  Sidhant Kaushik, Aditya Shah, Alex~S. Huang, Luke Zettlemoyer, Dieter Fox,
  Yu~Xiang, Anqi Li, Andreea Bobu, Abhishek Gupta, Stephen Tu, Erdem Biyik, and
  Jesse Zhang.
\newblock Robometer: Scaling general-purpose robotic reward models via
  trajectory comparisons.
\newblock \emph{arXiv preprint arXiv:2603.02115}, 2026.

\bibitem[Kim et~al.(2024)Kim, Pertsch, Karamcheti, Xiao, Balakrishna, Nair,
  Rafailov, Foster, Lam, Sanketi, et~al.]{kim2024openvla}
Moo~Jin Kim, Karl Pertsch, Siddharth Karamcheti, Ted Xiao, Ashwin Balakrishna,
  Suraj Nair, Rafael Rafailov, Ethan Foster, Grace Lam, Pannag Sanketi, et~al.
\newblock {OpenVLA}: An open-source vision-language-action model.
\newblock In \emph{Conference on Robot Learning}, 2024.

\bibitem[{Physical Intelligence} et~al.(2025){Physical Intelligence}, Black,
  Brown, et~al.]{black2025pi05}
{Physical Intelligence}, Kevin Black, Noah Brown, et~al.
\newblock $\pi_{0.5}$: a vision-language-action model with open-world
  generalization.
\newblock \emph{arXiv preprint arXiv:2504.16054}, 2025.

\bibitem[Bai et~al.(2026)Bai, Lyu, Zhou, Li, Wang, Xing, Zhao, Wang, Wang, Chi,
  Chen, and Zhang]{bai2026latent}
Shuanghao Bai, Jing Lyu, Wanqi Zhou, Zhe Li, Dakai Wang, Lei Xing, Xiaoguang
  Zhao, Pengwei Wang, Zhongyuan Wang, Cheng Chi, Badong Chen, and Shanghang
  Zhang.
\newblock Latent reasoning {VLA}: Latent thinking and prediction for
  vision-language-action models.
\newblock In \emph{International Conference on Machine Learning}, 2026.

\bibitem[Auer et~al.(2002)Auer, Cesa-Bianchi, and Fischer]{auer2002finite}
Peter Auer, Nicol{\`o} Cesa-Bianchi, and Paul Fischer.
\newblock Finite-time analysis of the multiarmed bandit problem.
\newblock \emph{Machine Learning}, 47:\penalty0 235--256, 2002.

\bibitem[Chen and Jiang(2019)]{chen2019information}
Jinglin Chen and Nan Jiang.
\newblock Information-theoretic considerations in batch reinforcement learning.
\newblock In \emph{International Conference on Machine Learning}, 2019.

\bibitem[Zhang(2002)]{zhang2002isotonic}
Cun-Hui Zhang.
\newblock Risk bounds in isotonic regression.
\newblock \emph{The Annals of Statistics}, 30\penalty0 (2):\penalty0 528--555,
  2002.

\bibitem[{OpenAI}(2024)]{openai2024mlebenchblog}
{OpenAI}.
\newblock {MLE-bench}: Evaluating machine learning agents on machine learning
  engineering.
\newblock \url{https://openai.com/index/mle-bench/}, 2024.

\end{thebibliography}

\clearpage
\appendix
\section*{Appendix}

\startcontents[appendix]
\printcontents[appendix]{}{1}{\setcounter{tocdepth}{2}}

\clearpage

\section{Notation}
\label{app:notation}

Table~\ref{tab:notation} collects the notation of the main text, in
order of first appearance.

\begin{table}[H]
\centering
\caption{\textbf{Notation used in the main text.}}
\label{tab:notation}
\setlength{\tabcolsep}{5pt}
\renewcommand{\arraystretch}{1.2}
\small
\begin{tabular}{l l}
\toprule
\textbf{Symbol} & \textbf{Meaning} \\
\midrule
$\tau$, $r(\tau)$ & a trajectory and its true score from real execution, $r(\tau)\in[0,1]$ \\
$\pi_\theta$, $J(\theta)$ & the agent policy and its expected score $\mathbb{E}_{\tau\sim\pi_\theta}[r(\tau)]$ \\
$n$, $A_i$, $s_i$ & group size, the advantage, and the policy gradient of trajectory $\tau_i$ \\
$\hat g$, $\gamma$, $T$ & the gradient estimate, the step size, and the number of training steps \\
$\hat r(\tau)$, $\hat A_i$ & the world model score and the advantage computed from it \\
$b(\tau)$, $\xi(\tau)$ & the bias and the zero-mean noise of the world model score \\
$B$, $\sigma$ & the maximal bias magnitude and the maximal noise standard deviation \\
$\mathcal{P}$ & the pool of score pairs $(\hat r_j, r_j)$ collected on anchor groups \\
$\hat f$ & the monotone recalibration map fit on $\mathcal{P}$ \\
$m$, $\mathcal{G}_E$, $\mathcal{G}_{WM}$ & the groups per step and the index sets of the anchor and world model ones \\
$g_E$, $g_{WM}$ & the sums of the group estimates $\hat g_k$ over $\mathcal{G}_E$ and over $\mathcal{G}_{WM}$ \\
$V_E$, $V_{WM}$ & the per-group gradient variances of the two streams \\
$\rho$, $c$ & the variance ratio $V_{WM}/V_{E} = 1 + c\,\sigma^{2}$ and its conversion constant \\
$\hat\eta^{2}$, $\hat\eta^{2}_{\mathrm{cal}}$ & the tracked residual of the calibrated scores and its warmup value \\
$J^\star$, $\Delta_0$ & the optimal score and the initial gap $J^\star - J(\theta_0)$ \\
$L$, $\mu$, $M$ & the smoothness, gradient domination, and gradient bound constants \\
$\varepsilon(T)$ & the convergence term of standard RL (Proposition~\ref{prop:standard}) \\
$T_0$ & the recalibration time scale of Theorem~\ref{thm:ours-conv} \\
\bottomrule
\end{tabular}
\end{table}

\section{Proofs}
\label{app:proofs}

This appendix proves every claim of Section~\ref{sec:theory}, in the
order of the story. Appendix~\ref{app:standard} derives the standard
convergence term $\varepsilon(T)$ of Remark~\ref{rem:price} for
training with real execution. Appendix~\ref{app:proof-wm} proves
Theorem~\ref{thm:wm-conv} for training on world model rewards, and
Appendix~\ref{app:proof-ours} proves Theorem~\ref{thm:ours-conv} for
\ours{} and concludes with the term-by-term comparison. The auxiliary
lemmas invoked along the way are only cited there, and their statements
with full proofs are collected in Appendix~\ref{app:lemmas}.
Throughout, the variance of a random vector $X$ is the scalar
$\mathrm{Var}(X) := \mathbb{E}\|X - \mathbb{E}X\|^{2}$, the trace of
its covariance matrix, which satisfies $\mathbb{E}\|X\|^{2} =
\|\mathbb{E}X\|^{2} + \mathrm{Var}(X)$.

\subsection{Warm-up: convergence of standard RL}
\label{app:standard}

We first derive the bound that training with real execution satisfies.
It contains no contribution from the world model and is exactly the
$\varepsilon(T)$ of Remark~\ref{rem:price}.

\begin{proposition}[Standard RL]
\label{prop:standard}
Under Assumption~\ref{ass:reg} and $\gamma \le 1/(8L)$, GRPO trained on
real execution rewards satisfies
\begin{equation*}
J^\star - \mathbb{E}[J(\theta_T)] \;\le\;
\Big(1-\tfrac{\gamma\mu}{4}\Big)^{T}\Delta_0
\;+\; \frac{4L\gamma}{\mu}\,V_{E}
\;=:\; \varepsilon(T).
\end{equation*}
\end{proposition}

\begin{proof}
We proceed in three steps.

\emph{Step 1, the estimate is clean.} Let $\mathbb{E}_t$ denote
expectation conditioned on $\theta_t$, so that $J(\theta_t)$ and its
gradient are fixed under $\mathbb{E}_t$. By
Lemma~\ref{lem:bridge}\,(i), proved in Appendix~\ref{app:lemmas}, the
estimator of Equation~\eqref{eq:grpo} satisfies $\mathbb{E}_t[\hat g_t]
= \kappa \nabla J(\theta_t)$ with $\kappa = 1 - \tfrac1n \in
[\tfrac12, 1)$ and $\mathrm{Var}(\hat g_t \mid \theta_t) \le
V_{E}$. Since $\|\mathbb{E}_t\hat g_t\| = \kappa\|\nabla
J(\theta_t)\| \le \|\nabla J(\theta_t)\|$, the second moment obeys
$\mathbb{E}_t\|\hat g_t\|^{2} = \|\mathbb{E}_t\hat g_t\|^{2} +
\mathrm{Var}(\hat g_t \mid \theta_t) \le \|\nabla J(\theta_t)\|^{2}
+ V_{E}$.

\emph{Step 2, one-step progress.} Write $\nabla := \nabla J(\theta_t)$.
By $L$-smoothness of $J$,
\begin{equation*}
J(\theta_{t+1}) \;\ge\; J(\theta_t) + \gamma\langle \nabla, \hat
g_t\rangle - \frac{L\gamma^{2}}{2}\,\|\hat g_t\|^{2}.
\end{equation*}
Taking expectations conditioned on $\theta_t$ and inserting Step 1,
\begin{equation*}
\mathbb{E}_t\, J(\theta_{t+1}) \;\ge\; J(\theta_t)
+ \gamma\Big(\kappa - \tfrac{L\gamma}{2}\Big)\|\nabla\|^{2}
- \frac{L\gamma^{2}}{2}\,V_{E}
\;\ge\; J(\theta_t) + \frac{\gamma}{8}\,\|\nabla\|^{2}
- \frac{L\gamma^{2}}{2}\,V_{E},
\end{equation*}
where the last step uses $\kappa \ge \tfrac12$ and $L\gamma \le
\tfrac18$.

\emph{Step 3, gradient domination and unrolling.} By
Assumption~\ref{ass:reg}, $\|\nabla\|^{2} \ge 2\mu\,(J^\star -
J(\theta_t))$. Writing $\delta_t := J^\star -
\mathbb{E}[J(\theta_t)]$ and taking total expectations,
\begin{equation*}
\delta_{t+1} \;\le\; \Big(1-\tfrac{\gamma\mu}{4}\Big)\,\delta_t
+ \frac{L\gamma^{2}}{2}\,V_{E}.
\end{equation*}
Unrolling over $t = 0,\dots,T-1$ and bounding the geometric sum by
$\sum_{k\ge0}(1-\tfrac{\gamma\mu}{4})^{k} = \tfrac{4}{\gamma\mu}$ gives
$\delta_T \le (1-\tfrac{\gamma\mu}{4})^{T}\Delta_0 +
\tfrac{2L\gamma}{\mu}V_{E}$, and we state the looser constant
$\tfrac{4L\gamma}{\mu}$ to share constants with the perturbed case
below.
\end{proof}

\subsection{Proof of Theorem~\ref{thm:wm-conv}}
\label{app:proof-wm}

With predicted rewards the estimate is no longer clean. By
Lemma~\ref{lem:bridge}, stated and proved in
Appendix~\ref{app:lemmas}, the estimator computed with predicted
scores decomposes as $\hat g_k^{WM} = \hat g_k + g_b +
g_\xi$, where $g_b$ is a bounded deterministic shift induced by the
bias and $g_\xi$ is a zero-mean fluctuation induced by the noise. The
recursion of Proposition~\ref{prop:standard} then goes through with
two modifications, the variance grows and a perturbation term appears,
which is the content of Lemma~\ref{lem:descent} in
Appendix~\ref{app:lemmas}.

\begin{restatethm}[Convergence with world model rewards, Theorem~\ref{thm:wm-conv}
of Section~\ref{subsec:theory-cost} with explicit constants]
Under Assumption~\ref{ass:reg}, GRPO trained on world model rewards
with $\gamma \le 1/(8L)$ satisfies
\begin{equation*}
J^\star - \mathbb{E}[J(\theta_T)] \;\le\;
\Big(1-\tfrac{\gamma\mu}{4}\Big)^{T}\Delta_0
\;+\; \frac{32M^{2}}{\mu}\,B^{2}
\;+\; \frac{4L\gamma}{\mu}\,V_{WM}.
\end{equation*}
\end{restatethm}

\begin{proof}
The plan is to verify the hypotheses of Lemma~\ref{lem:descent} for the
decomposition supplied by Lemma~\ref{lem:bridge}, and then read off the
bound. By Lemma~\ref{lem:bridge}, the estimate computed with predicted
scores splits as $u + d$ with $u := \hat g_k + g_\xi$ and $d := g_b$.
The part $u$ is the clean estimate plus a conditionally zero-mean
fluctuation, so its mean is the unchanged $(1-\tfrac1n)\nabla J$, and
since $\hat g_k$ and $g_\xi$ are uncorrelated their variances add,
$V_{E} + 4M^{2}\sigma^{2}/n = V_{WM}$. The part $d$
obeys $\|d\| \le 2MB$ almost surely, which serves as the constant
perturbation bound $D$. Applying Lemma~\ref{lem:descent} with
$\kappa_t = 1-\tfrac1n \ge \tfrac12$, $V = V_{WM}$ and $D =
2MB$ then yields the three terms of the claim, the geometric decay,
the constant bias term $\tfrac{8}{\mu}(2MB)^{2} =
\tfrac{32M^{2}}{\mu}B^{2}$, and the variance term
$\tfrac{4L\gamma}{\mu}V_{WM}$.
\end{proof}

\subsection{Proof of Theorem~\ref{thm:ours-conv} and the comparison}
\label{app:proof-ours}

\phantomsection\label{app:conditions}
Compared with Appendix~\ref{app:proof-wm}, two things improve for
\ours{}, the recalibration makes the perturbation bound shrink over
time, and the fusion makes the variance harmonic. Analyzing the
recalibration needs one modeling assumption beyond
Assumption~\ref{ass:reg}.

\begin{assumption}[Score scale]
\label{ass:scale}
The bias acts on the score scale as a monotone distortion, that is,
$b(\tau) = \phi(r(\tau)) - r(\tau)$ for a non-decreasing $\phi$, and
the noise $\xi$ is independent across trajectories with $|\xi| \le 1$.
\end{assumption}

This is what makes the bias learnable by a monotone fit. How fast it
is learned is quantified by Lemma~\ref{lem:recal}, whose constant
$c_f$ measures how quickly the recalibration error decays. The constant $T_0$ of Theorem~\ref{thm:ours-conv} is
$T_0 := c_f^{2}/B^{2}$, the number of steps after which the residual
bias of the recalibrated scores falls below the raw bias of the world
model.

\begin{restatethm}[Convergence of \ours{}, Theorem~\ref{thm:ours-conv} of
Section~\ref{subsec:theory-fix} with explicit constants]
Under Assumptions~\ref{ass:reg} and~\ref{ass:scale} and $\gamma \le
1/(8L)$, \ours{} satisfies with probability at least $1-\delta$
\begin{equation*}
\begin{aligned}
J^\star - \mathbb{E}[J(\theta_T)] \;\le\;\;
&\Big(1-\tfrac{\gamma\mu}{4}\Big)^{T}\Delta_0
\;+\; \frac{64\,c_f^{2}M^{2}}{\mu}\cdot\frac{\log(2KT/\delta)}{T}\\[2pt]
&+\; 4\gamma T M^{2}B^{2}\Big(1-\tfrac{\gamma\mu}{4}\Big)^{T/2}
\;+\; \frac{4L\gamma}{\mu}
\Big(\frac{1}{V_{E}}+\frac{1}{V_{WM}}\Big)^{-1},
\end{aligned}
\end{equation*}
where $V_{WM}$ is formed with the post-recalibration noise and
$c_f$ is the constant of Lemma~\ref{lem:recal}. The third term decays
exponentially in $T$, and the whole perturbation is at most
$\tfrac{8}{\mu}M^{2}B^{2}$ at every $T$. Since $\min(B^{2},
c_f^{2}/T)$ and $B^{2}/(1+T/T_0)$ with $T_0 := c_f^{2}/B^{2}$ agree up
to a factor of two, this is the form stated in
Section~\ref{subsec:theory-fix}.
\end{restatethm}

\begin{proof}
The proof has two parts. We first verify that the fused estimate
satisfies the hypotheses of Lemma~\ref{lem:descent}, now with a smaller
variance and a perturbation that shrinks over time, and we then bound
the resulting perturbation sum, which is where the contraction of the
bias comes from.

\emph{Part 1, the hypotheses.} Work on the success event of
Lemma~\ref{lem:recal}, which has probability at least $1-\delta$. On
this event the recalibrated scores have systematic error at most
$\min\big(B,\, c_f\sqrt{\log(2KT/\delta)/t}\big)$ at step $t$, since
recalibration maps into the score range, so the error never exceeds
$B$, and the sharper bound of Lemma~\ref{lem:recal} applies once enough anchor pairs have accumulated. Applying Lemma~\ref{lem:bridge} to the recalibrated
scores decomposes the world model estimate as $u_{WM} +
d_{WM}$ with $\mathrm{Var}(u_{WM}) \le
V_{WM}$ and $\|d_{WM}\| \le 2M \min(B,
c_f\sqrt{\log(2KT/\delta)/t})$, while the anchor estimate has the
same mean $(1-\tfrac1n)\nabla J$ with variance $V_{E}$, and
the two are independent as they are computed on disjoint groups. By
Lemma~\ref{lem:ivw}, fusing them with inverse-variance weights attains
the harmonic variance, and since the weights sum to one the fused
estimate keeps the common mean, so Lemma~\ref{lem:descent} applies
with $\kappa_t = 1-\tfrac1n$, $V =
(1/V_{E}+1/V_{WM})^{-1}$ and $D_t = 2M\min(B,
c_f\sqrt{\log(2KT/\delta)/t})$, which already produces the geometric
term and the variance term of the claim.

\emph{Part 2, the perturbation sum.} What remains is the middle term
of Lemma~\ref{lem:descent},
\begin{equation*}
2\gamma\sum_{t<T}\Big(1-\tfrac{\gamma\mu}{4}\Big)^{T-1-t}D_t^{2},
\qquad
D_t^{2} \;\le\; 4M^{2}\min\Big(B^{2},\;
c_f^{2}\tfrac{\log(2KT/\delta)}{t}\Big).
\end{equation*}
The two factors decay in opposite directions, the geometric weight is
small for early $t$ and the summand $D_t^{2}$ is small for late $t$,
so we split the sum at $T/2$ and use the stronger effect on each half.
For $t \ge T/2$ the summand is small, $D_t^{2} \le
8M^{2}c_f^{2}\log(2KT/\delta)/T$, and the geometric weights sum to at
most $\tfrac{4}{\gamma\mu}$, so this half contributes at most
$\tfrac{64c_f^{2}M^{2}}{\mu}\log(2KT/\delta)/T$, the contracting bias
term. For $t < T/2$ the weight is small, $(1-\tfrac{\gamma\mu}{4})^{T-1-t}
\le (1-\tfrac{\gamma\mu}{4})^{T/2}$, and with the crude bound $D_t^{2}
\le 4M^{2}B^{2}$ over at most $T/2$ terms this half contributes at most
$4\gamma T M^{2}B^{2}(1-\tfrac{\gamma\mu}{4})^{T/2}$, which vanishes
exponentially in $T$. Finally, bounding the same sum with the constant
$D \equiv 2MB$ instead gives $\tfrac{8}{\mu}M^{2}B^{2}$, so the
perturbation never exceeds the floor of Theorem~\ref{thm:wm-conv} and
decays as $\widetilde{O}(M^{2}c_f^{2}/T)$, which is the contracting
form stated in Section~\ref{subsec:theory-fix}.
\end{proof}

\begin{corollary}
\label{cor:improve}
Assume in addition that the recalibration map is non-expansive, so that
the post-recalibration noise satisfies $\eta^{2} \le \sigma^{2}$. Then,
comparing Theorem~\ref{thm:ours-conv} with Theorem~\ref{thm:wm-conv}
term by term, the leading terms coincide, there exists a finite $T_1$
such that for all $T \ge T_1$ the bias term of \ours{} is strictly
smaller than the floor $\tfrac{32M^{2}}{\mu}B^{2}$ and converges to
zero as $T \to \infty$, and the variance term of \ours{} is smaller by
the factor $1 + V_{WM}/V_{E} > 1$.
\end{corollary}

\begin{proof}
The leading terms are identical by construction. For the bias, the
term of Theorem~\ref{thm:ours-conv} falls below
$\tfrac{32M^{2}}{\mu}B^{2}$ as soon as $\log(2KT/\delta)/T \le
B^{2}/(2c_f^{2})$ and the exponentially decaying remainder is below the
same threshold, both of which hold for all $T$ beyond some finite
$T_1$, and the term itself converges to zero. For the variance,
Lemma~\ref{lem:ivw} writes the fused variance as
$V_{WM}/(1 + V_{WM}/V_{E})$, and
non-expansiveness gives $\eta^{2} \le \sigma^{2}$, so the
$V_{WM}$ of Theorem~\ref{thm:ours-conv}, formed with the
post-recalibration noise $\eta$, is no larger than the
$V_{WM}$ of Theorem~\ref{thm:wm-conv}, formed with
$\sigma$.
\end{proof}

\begin{remark}[Plug-in weights]
\label{rem:plugin}
The bounds above are stated with the exact variances, while training
uses the plug-in ratio of Section~\ref{subsec:anchor}, whose varying
part is measured online and whose scale is fixed once by the
normalization $c := 1/\hat\eta^{2}_{\mathrm{cal}}$. A misspecified
scale would generically cost a first-order term. It does not here,
because the inverse-variance weights sit at the minimum of the fused
variance, where the first-order term vanishes identically. Writing
$\rho := V_{WM}/V_{E}$ for the true ratio, $\hat\rho$ for the one used
and $\varepsilon_\rho := (\hat\rho-\rho)/\rho$ for its relative error, the
last claim of Lemma~\ref{lem:ivw} with $\alpha = \rho/(1+\rho)$ becomes
\begin{equation*}
\frac{\mathrm{Var}(\hat g_{\hat\rho})}{\mathrm{Var}(\hat g_{\rho})}
\;=\; 1 \;+\; \frac{\rho}{(1+\rho)^{2}}\,\varepsilon_\rho^{2}
\;+\; O(\varepsilon_\rho^{3}),
\qquad \frac{\rho}{(1+\rho)^{2}} \;\le\; \frac{1}{4},
\end{equation*}
so a relative error in the ratio inflates the variance by at most
$\varepsilon_\rho^{2}/4$ to this order. The scales measured in our runs
(Appendix~\ref{app:experiment_details}) put this excess below three
percent at the calibration point, so the variance term of
Theorem~\ref{thm:ours-conv} is essentially unaffected.
\end{remark}

\subsection{Auxiliary lemmas}
\label{app:lemmas}

This subsection collects the four lemmas used above. Before each
statement we recall why the lemma is needed and where it is consumed.

The first lemma is the engine behind both theorems. It answers how far
gradient ascent can still progress when the estimate carries a bounded
perturbation and extra variance, and every convergence statement of
this paper is an instance of it, Proposition~\ref{prop:standard} with
$D_t \equiv 0$, Theorem~\ref{thm:wm-conv} with $(V, D) =
(V_{WM},\, 2MB)$, and Theorem~\ref{thm:ours-conv} with the
harmonic variance and a shrinking $D_t$. The middle term of its bound
is the weighted telescoping sum that the two theorems then control in
different ways.

\begin{lemma}[Perturbed ascent under gradient domination]
\label{lem:descent}
Let $\theta_{t+1} = \theta_t + \gamma \hat g_t$ with $\gamma \le
1/(8L)$. Suppose that, conditioned on $\theta_t$, the estimate
decomposes as $\hat g_t = u_t + d_t$, where $\mathbb{E}[u_t \mid
\theta_t] = \kappa_t \nabla J(\theta_t)$ with
$\kappa_t\in[\tfrac12,1]$, $\mathrm{Var}(u_t \mid \theta_t) \le V$,
and $\|d_t\| \le D_t$ almost surely. Then
\begin{equation*}
J^\star - \mathbb{E}[J(\theta_T)]
\;\le\;
\Big(1-\tfrac{\gamma\mu}{4}\Big)^{T}\Delta_0
\;+\; 2\gamma \sum_{t=0}^{T-1}\Big(1-\tfrac{\gamma\mu}{4}\Big)^{T-1-t}
D_t^{2}
\;+\; \frac{4L\gamma}{\mu}\,V .
\end{equation*}
In particular, if $D_t \equiv D$, the middle term is at most
$\tfrac{8}{\mu} D^{2}$.
\end{lemma}

\begin{proof}
The plan mirrors the three steps of Proposition~\ref{prop:standard}.
We lower-bound the useful progress $\langle\nabla,\hat g_t\rangle$,
upper-bound the harmful second moment $\|\hat g_t\|^{2}$, and combine
the two through smoothness into a recursion for $\delta_t$. Write
$\nabla := \nabla J(\theta_t)$ and let $\mathbb{E}_t$ denote
expectation conditioned on $\theta_t$; under $\mathbb{E}_t$ the
quantities $J(\theta_t)$ and $\nabla$ are fixed, so only the terms
containing $\hat g_t$ are affected. By $L$-smoothness,
$J(\theta_{t+1}) \ge J(\theta_t) + \gamma\langle \nabla, \hat
g_t\rangle - \frac{L\gamma^{2}}{2}\|\hat g_t\|^{2}$.

\emph{Step 1, the progress term.} Splitting $\hat g_t = u_t + d_t$,
the inner product separates into signal and contamination,
$\mathbb{E}_t\langle \nabla,\hat g_t\rangle = \kappa_t\|\nabla\|^{2} +
\langle \nabla, \mathbb{E}_t d_t\rangle$. The contamination can only
be bounded by Cauchy--Schwarz, $\langle \nabla, \mathbb{E}_t
d_t\rangle \ge -\|\nabla\| D_t$, and a product of two different
quantities is inconvenient in a recursion, so we separate it with
Young's inequality $ab \le a^{2}/4 + b^{2}$,
\begin{equation*}
\mathbb{E}_t\langle \nabla,\hat g_t\rangle
\ge \kappa_t\|\nabla\|^{2} - \|\nabla\| D_t
\ge \big(\kappa_t - \tfrac14\big)\|\nabla\|^{2} - D_t^{2}.
\end{equation*}
The choice of $\tfrac14$ sacrifices a quarter of the signal to turn
the cross term into a pure $D_t^{2}$.

\emph{Step 2, the second moment.} We control mean and fluctuation
separately. The mean obeys $\|\mathbb{E}_t \hat g_t\| \le
\kappa_t\|\nabla\| + D_t$, so $(a+b)^{2} \le 2a^{2}+2b^{2}$ gives
$\|\mathbb{E}_t\hat g_t\|^{2} \le 2\|\nabla\|^{2} + 2D_t^{2}$, and the
fluctuation obeys $\mathrm{Var}(\hat g_t) \le 2\,\mathrm{Var}(u_t) +
2\,\mathrm{Var}(d_t) \le 2V + 2D_t^{2}$. Adding the two,
$\mathbb{E}_t\|\hat g_t\|^{2} \le 2\|\nabla\|^{2} + 4D_t^{2} + 2V$.

\emph{Step 3, combine and unroll.} Substituting the two bounds into
the smoothness inequality, and simplifying with $\kappa_t \ge
\tfrac12$ and $L\gamma \le \tfrac18$,
\begin{align*}
\mathbb{E}_t J(\theta_{t+1})
&\ge J(\theta_t)
+ \gamma\big(\kappa_t - \tfrac14 - L\gamma\big)\|\nabla\|^{2}
- \gamma D_t^{2}(1 + 2L\gamma) - L\gamma^{2}V\\
&\ge J(\theta_t) + \tfrac{\gamma}{8}\|\nabla\|^{2}
- 2\gamma D_t^{2} - L\gamma^{2}V .
\end{align*}
Gradient domination replaces the gradient by the gap,
$\|\nabla\|^{2} \ge 2\mu\,(J^\star - J(\theta_t))$, and taking total
expectations through the tower property turns the display into a
recursion on the deterministic sequence $\delta_t := J^\star -
\mathbb{E}[J(\theta_t)]$,
\begin{equation*}
\delta_{t+1} \;\le\; \Big(1-\tfrac{\gamma\mu}{4}\Big)\,\delta_t
+ 2\gamma D_t^{2} + L\gamma^{2}V .
\end{equation*}
Multiplying the step-$t$ inequality by $(1-\tfrac{\gamma\mu}{4})^{T-1-t}$
and summing over $t = 0,\dots,T-1$, the intermediate $\delta_t$ cancel
telescopically and the claim follows. For constant $D$ the geometric
sum $\sum_{k\ge 0}(1-\tfrac{\gamma\mu}{4})^{k} = \tfrac{4}{\gamma\mu}$
bounds the middle term by $\tfrac{8}{\mu}D^{2}$ and the last term by
$\tfrac{4L\gamma}{\mu}V$.
\end{proof}

The second lemma is what allows Lemma~\ref{lem:descent} to be fed
with world model rewards. It converts the reward-level error model of
Equation~\eqref{eq:error-structure} into exactly the three quantities
that Lemma~\ref{lem:descent} consumes, the mean of the clean part,
its variance, and the almost-sure size of the perturbation, and it
also produces the variance relation $V_{WM} =
(1+c\,\sigma^{2})V_{E}$ used in
Section~\ref{subsec:anchor}.

\begin{lemma}[From reward error to gradient error]
\label{lem:bridge}
Let $\hat g_k$ be computed with true scores and $\hat g_k^{WM}$
with predicted scores. Under
Assumptions~\ref{ass:reg} and~\ref{ass:scale},
$\hat g_k^{WM} = \hat g_k + g_b + g_\xi$, where
\begin{enumerate}[label=(\roman*),leftmargin=24pt]
\item $\mathbb{E}[\hat g_k] = \big(1-\tfrac{1}{n}\big)\nabla J(\theta)$
and $V_{E} = \mathrm{Var}(\hat g_k) \le M^{2}$;
\item $g_b = \frac{1}{n}\sum_i (b_i - \bar b)\,s_i$ satisfies
$\|g_b\| \le 2MB$ almost surely, and vanishes whenever $b$ is constant
within the group;
\item $g_\xi$ has zero conditional mean, is uncorrelated with $\hat
g_k$, and satisfies $\mathbb{E}\|g_\xi\|^{2} \le 4M^{2}\sigma^{2}/n$,
so that $V_{WM} = V_{E} + 4M^{2}\sigma^{2}/n = (1 +
c\,\sigma^{2})\,V_{E}$ with $c = 4M^{2}/(n V_{E})$,
the relation used in Section~\ref{subsec:anchor}.
\end{enumerate}
\end{lemma}

\begin{proof}
Recall from Equation~\eqref{eq:ivw} that $\hat g_k =
\frac{1}{n}\sum_i A_i s_i$, where the advantage $A_i = r_i -
\frac{1}{n}\sum_j r_j$ is a linear function of the scores of its
group. Abbreviate $r_i := r(\tau_i)$, $b_i := b(\tau_i)$, and $\xi_i
:= \xi(\tau_i)$.

\emph{The decomposition is pure algebra.} Replacing every true score
by the predicted one substitutes into the same linear formula, so with
$\hat r_i = r_i + b_i + \xi_i$ the new advantage is
\begin{equation*}
\hat A_i \;=\; \hat r_i - \frac{1}{n}\sum_j \hat r_j
\;=\; A_i + (b_i - \bar b) + (\xi_i - \bar\xi),
\end{equation*}
with $\bar b$, $\bar\xi$ the group means. Multiplying by $s_i/n$ and
summing over $i$ gives $\hat g_k^{WM} = \hat g_k + g_b +
g_\xi$ with $g_b := \frac{1}{n}\sum_i (b_i-\bar b)s_i$ and $g_\xi :=
\frac{1}{n}\sum_i (\xi_i - \bar\xi)\,s_i$; no probabilistic argument
is involved, the three parts simply collect the score, bias, and noise
contributions. It remains to control each part in the form that
Lemma~\ref{lem:descent} consumes, the mean and variance of the clean
part, the almost-sure size of the bias part, and the second moment of
the noise part.

(i) \emph{The clean part.} Expanding the advantage,
$\mathbb{E}[A_i s_i] = \mathbb{E}[r_i s_i] - \frac{1}{n}\sum_j
\mathbb{E}[r_j s_i]$. For $j \ne i$ the two trajectories are
independent given the task, so the expectation factorizes, and
$\mathbb{E}[s_i] = \mathbb{E}[\nabla\log\pi_\theta] = 0$ kills the
term; the $j = i$ term contributes $\frac{1}{n}\mathbb{E}[r_i s_i]$.
Summing over $i$,
$\mathbb{E}[\hat g_k] =
(1-\tfrac1n)\,\mathbb{E}[r(\tau)\nabla\log\pi_\theta(\tau)] =
(1-\tfrac1n)\nabla J(\theta)$, where the last step is the policy
gradient identity. For the variance, $r_i \in [0,1]$ forces $|A_i|
\le 1$, hence $\|\hat g_k\| \le \frac{1}{n}\sum_i\|s_i\| \le M$, and a
vector bounded by $M$ has variance at most $\mathbb{E}\|\hat
g_k\|^{2} \le M^{2}$, so $V_{E} \le M^{2}$.

(ii) \emph{The bias part.} It is deterministic given the
trajectories, so what Lemma~\ref{lem:descent} needs is an almost-sure
bound. Since $|b_i| \le B$, the centered value obeys $|b_i - \bar b|
\le 2B$, and therefore $\|g_b\| \le \frac{1}{n}\sum_i |b_i - \bar
b|\,\|s_i\| \le 2MB$. If $b$ is constant within the group, the
centering removes it entirely, $b_i - \bar b = 0$, which is why only
the within-group variation of the bias ever reaches the gradient.

(iii) \emph{The noise part.} It should act as extra variance, so we
verify that it is conditionally zero-mean and compute its second
moment. Rewrite $g_\xi = \frac{1}{n}\sum_i \xi_i (s_i - \bar s)$ with
$\bar s := \frac1n\sum_j s_j$. Conditioned on the trajectories the
$\xi_i$ are independent and zero-mean, so $\mathbb{E}[g_\xi \mid
\tau_{1:n}] = 0$ and all cross terms of $\mathbb{E}\|g_\xi\|^{2}$
vanish, leaving
\begin{equation*}
\mathbb{E}\|g_\xi\|^{2} = \frac{1}{n^{2}}\sum_i
\mathbb{E}\big[\mathbb{E}[\xi_i^{2}\mid\tau]\,\|s_i-\bar s\|^{2}\big]
\le \frac{\sigma^{2}}{n^{2}}\cdot n \cdot (2M)^{2}
= \frac{4M^{2}\sigma^{2}}{n}.
\end{equation*}
The same conditional zero mean makes $g_\xi$ uncorrelated with $\hat
g_k$, so the two variances add, $V_{WM} = V_{E} +
4M^{2}\sigma^{2}/n = (1 + c\,\sigma^{2})\,V_{E}$ with $c =
4M^{2}/(n V_{E})$, the stated relation.
\end{proof}

The third lemma supplies the shrinking perturbation bound $D_t$ used
in the proof of Theorem~\ref{thm:ours-conv}. It quantifies how fast
the isotonic recalibration of Equation~\eqref{eq:isotonic} learns the
distortion $\phi$ from the accumulating anchor pairs, and its
$1/\sqrt{t}$ rate is what ultimately becomes the contraction of the
bias term.

\begin{lemma}[Recalibration error]
\label{lem:recal}
Under Assumption~\ref{ass:scale}, with probability at least $1-\delta$
simultaneously for all $1 \le t \le T$, the recalibration map $\hat
f_t$ is well defined once enough anchor pairs have accumulated, and
the systematic error of the recalibrated score obeys
\begin{equation*}
\big|\mathbb{E}[\hat f_t(\hat r)\mid\tau] - r(\tau)\big|
\;\le\; c_f\sqrt{\log(2KT/\delta)/t},
\end{equation*}
where $K$ is the number of distinct scores the benchmark produces and
the constant $c_f$ is given in Equation~\eqref{eq:cf} of the proof.
\end{lemma}

\begin{proof}
We first fix notation for the scores our benchmarks produce. The
leaderboard percentile is computed against a finite leaderboard and
the VLA reward is binary, so the true scores take finitely many values
$\ell_1 < \dots < \ell_K$, and the recalibration of
Equation~\eqref{eq:isotonic} fits one vertex per level. Write
$\Delta_\ell := \max_k(\ell_{k+1}-\ell_k)$ for the largest gap between
adjacent levels, $\Delta_\phi := \min_k
\big(\phi(\ell_{k+1})-\phi(\ell_k)\big)$ for the smallest separation
the distortion leaves between them, $q$ for the smallest fraction of
the anchor pairs that any level receives, and $a$ for the number of
new pairs per step. The rate below is informative when $\Delta_\phi > 0$
and $q > 0$, that is, when the distortion collapses no two levels onto
one another and no level is starved of anchor pairs, which is the
discrete-level analogue of the arm gaps of bandit
analysis~\citep{auer2002finite} and of the coverage conditions of
off-policy evaluation~\citep{chen2019information}, and is what turns
isotonic regression from a worst-case problem into one with a
parametric rate~\citep{zhang2002isotonic}. In terms of these
quantities the constant of the statement is
\begin{equation}
c_f \;:=\; \frac{2\Delta_\ell}{\Delta_\phi}\sqrt{\frac{1}{2qa}} .
\label{eq:cf}
\end{equation}

The proof then has three stages that mirror the recalibration
pipeline. We first show that the empirical mean of the predictions at
each score level concentrates, then that the isotonic projection can
only inherit this accuracy, and finally that the interpolated map
turns accurate vertices into accurate calibrated scores.

\emph{Stage 1, level means concentrate.} Coverage is what makes this
stage possible, as a level that received no pairs could never be
calibrated. At step $t$ each level holds at least $qat$ pairs, each of the form
$\hat r = \phi(\ell_k) + \xi$ with $\xi$ zero-mean and $|\xi| \le
1$, so Hoeffding's inequality gives $\Pr(|\hat\mu_k - \phi(\ell_k)|
\ge s) \le 2e^{-2qats^{2}}$. The guarantee is needed simultaneously
for every level and every step, so we take a union bound over the $K$
levels and $T$ steps, and the choice $s = e(t) := \sqrt{\log(2KT/\delta)/(2qat)}$ keeps the total
failure probability below $\delta$.

\emph{Stage 2, the projection does not hurt.} The isotonic fit
replaces the raw means $\hat\mu_k$ by the closest monotone sequence
$\hat\phi_k$, and we must check that this projection cannot push an
estimate away from the truth. The min--max representation $\hat\phi_k
= \min_{v \ge k}\max_{u \le k} \mathrm{avg}(\hat\mu_u,\dots,\hat\mu_v)$
expresses every fitted value as an average of raw means, each within
$e(t)$ of its true value. For any $v \ge k$, monotonicity of $\phi$
gives $\max_{u\le k}\mathrm{avg}[u,v] \ge \mathrm{avg}[k,v] \ge
\mathrm{avg}(\phi(\ell_k),\dots,\phi(\ell_v)) - e(t) \ge \phi(\ell_k) -
e(t)$, and taking the minimum over $v$ preserves the bound; choosing
$v = k$ gives symmetrically $\hat\phi_k \le \phi(\ell_k) + e(t)$. So
the fitted vertices are as accurate as the raw means.

\emph{Stage 3, from vertices to calibrated scores.} Separation is what
makes this stage possible, as levels the distortion has collapsed onto
one another could not be told apart by any monotone fit. When $e(t) \le
\Delta_\phi/4$, consecutive fitted vertices stay separated,
$\hat\phi_{k+1}-\hat\phi_k \ge \Delta_\phi - 2e(t) \ge \Delta_\phi/2$, so the
piecewise-linear interpolation $\hat f_t$ is well defined and its
Lipschitz constant is at most $\Delta_\ell/(\Delta_\phi/2) =
2\Delta_\ell/\Delta_\phi$. A trajectory with true score $\ell_k$ produces
predictions with conditional mean $\phi(\ell_k)$, and $\hat f_t$ maps
the fitted vertex $\hat\phi_k$ exactly to $\ell_k$, so the systematic
error of the calibrated score is at most the Lipschitz constant times
the vertex error, $|\hat f_t(\phi(\ell_k)) - \ell_k| = |\hat
f_t(\phi(\ell_k)) - \hat f_t(\hat\phi_k)| \le
\tfrac{2\Delta_\ell}{\Delta_\phi}e(t)$, which is the claim.
\end{proof}

The last lemma justifies the fusion rule of Equation~\eqref{eq:ivw}
and supplies the variance $V$ used in the proof of
Theorem~\ref{thm:ours-conv}. Its final claim, that the fused variance
is flat around the optimal weight, is what Remark~\ref{rem:plugin}
invokes to argue robustness to the estimated ratio.

\begin{lemma}[Optimal linear fusion]
\label{lem:ivw}
Let $g_{1}$ and $g_{2}$ be independent
unbiased estimates of the same quantity with variances
$V_{1}$ and $V_{2}$. Among all their unbiased
linear combinations, that is, with weights summing to one, the
variance is minimized by weights proportional to
$V_{1}^{-1}$ and $V_{2}^{-1}$, and the minimal
value is
\begin{equation*}
\Big(\frac{1}{V_{1}}+\frac{1}{V_{2}}\Big)^{-1}
\;<\; \min(V_{1},\,V_{2}).
\end{equation*}
Moreover, for any weight $\alpha$ on $g_{1}$,
$\mathrm{Var}(\hat g_\alpha) = \mathrm{Var}(\hat g_{\alpha^\star}) +
(V_{1}+V_{2})(\alpha-\alpha^\star)^{2}$, so a
misspecified weight costs only a second-order term.
\end{lemma}

\begin{proof}
Write the combination with a single weight, $\hat g_\alpha = \alpha
g_{1} + (1-\alpha)g_{2}$, which is
unbiased for every $\alpha$, so the choices differ only in variance.
By independence, $\mathrm{Var}(\hat g_\alpha) =
\alpha^{2}V_{1} + (1-\alpha)^{2}V_{2}$, a strictly
convex quadratic in $\alpha$, so the unique minimizer is found by
setting the derivative to zero, $\alpha^\star =
V_{2}/(V_{1}+V_{2})$, that is, weights
proportional to the inverse variances. Substituting $\alpha^\star$
back gives
$V_{1}V_{2}/(V_{1}+V_{2})$,
which equals the stated harmonic form and is strictly smaller than
each of $V_{1}$ and $V_{2}$ because both
precisions are positive. Finally, expanding the quadratic around its
minimizer gives $\mathrm{Var}(\hat g_\alpha) = \mathrm{Var}(\hat
g_{\alpha^\star}) + (V_{1}+V_{2})
(\alpha-\alpha^\star)^{2}$, since its second derivative is
$2(V_{1}+V_{2})$; this is the second-order
insensitivity invoked in Remark~\ref{rem:plugin}.
\end{proof}

\begin{remark}[On restricting to linear combinations]
\label{rem:blue}
Lemma~\ref{lem:ivw} optimizes within the linear family, which is the
weakest requirement that suits our setting, as it uses only the second
moments of the two estimates and preserves unbiasedness whatever their
distribution. If the estimates are in addition taken to be Gaussian,
which group averages approach as the group size grows, the same
inverse-variance weights are minimum-variance among all unbiased
estimators, so the linear restriction is not what limits the result.
\end{remark}

\section{Choice of the Task Pool}
\label{app:mle-issues}

Section~\ref{subsec:exp-setup} builds the task pool from MLE-Dojo
rather than from MLE-Bench directly, for three reasons.

First, part of the original pool is no longer usable. Two of the
underlying Kaggle competitions,
\path{detecting-insults-in-social-commentary} and
\path{the-icml-2013-whale-challenge-right-whale-redux}, have been
permanently closed, and their official data downloads and submission
entries are gone.

Second, the medal-based metric fails to measure model ability.
MLE-Bench scores an agent by the number of gold, silver, and bronze
medals. At the model scales we study, an agent earns medals on only
one or two fixed competitions, so the medal count stays flat no matter
how much the agent improves on the rest of the pool. The counts are
also skewed, with more golds than silvers or bronzes in published
evaluations~\citep{nathani2025mlgym}, because a handful of tasks admit
near-perfect accuracy and hand out gold medals while the remaining
tasks hand out nothing. The ranking against human competitors is what
reflects ability, so our protocol reports the leaderboard percentile,
which credits improvement on every task.

Third, contamination. The competitions and many of their winning
solutions predate the pretraining corpora of current models, a risk
acknowledged both in the original release~\citep{chan2024mlebench} and
in the accompanying {OpenAI} blog post~\citep{openai2024mlebenchblog},
and the public evaluation has not been actively maintained since.

MLE-Dojo wraps a superset of the same competitions in an interactive
environment with live grading, which avoids the dead entries, exposes
leaderboard percentile directly, and allows a free re-partition of the
pool. Manually splitting a public pool into train and test portions
follows common practice, and our split is fixed once and shared by all
configurations.

\paragraph{Why these training and test sets.}
The partition follows one deterministic rule, fixed before any
training, and is designed to keep the categories balanced on both
sides. Within each of the three categories (tabular, text, image), the
20 smallest competitions by dataset size are selected so that every
task fits the execution sandbox, the 15 smallest of them form the
training set, and the remaining 5 form the held-out set. Audio
competitions are excluded because they do not fit the sandbox, and one
held-out task (\path{billion-word-imputation}) is dropped because its
grader depends on a package unavailable in our environment, leaving 45
training and 14 held-out competitions. The rule has a useful side
effect, as every held-out task is larger than every training task of
its category, so the held-out evaluation also measures extrapolation
from small training tasks to larger unseen ones. We also verified that no task appears on both
sides.\footnote{One held-out task shares the Jigsaw competition family
with a training task, with different data and a different metric.}
The data of every competition was cached locally in advance, so
training and grading never depend on the Kaggle servers.
Table~\ref{tab:train-tasks} lists the training competitions and
Table~\ref{tab:heldout-tasks} the held-out ones.

\begin{table}[h]
\centering
\caption{\textbf{The 45 training competitions of MLE-Dojo (train), by
category.}}
\label{tab:train-tasks}
\setlength{\tabcolsep}{4pt}
\renewcommand{\arraystretch}{1.2}
\footnotesize
\begin{tabular}{>{\raggedright\arraybackslash}p{4.6cm} >{\raggedright\arraybackslash}p{4.6cm} >{\raggedright\arraybackslash}p{4.6cm}}
\toprule
\textbf{Tabular} & \textbf{Text} & \textbf{Image} \\
\midrule
\path{mercedes-benz-greener-manufacturing} & \path{kaggle-llm-science-exam} & \path{aerial-cactus-identification} \\
\path{icr-identify-age-related-conditions} & \path{movie-review-sentiment-analysis-kernels-only} & \path{leaf-classification} \\
\path{kobe-bryant-shot-selection} & \path{llm-detect-ai-generated-text} & \path{denoising-dirty-documents} \\
\path{home-data-for-ml-course} & \path{random-acts-of-pizza} & \path{facial-keypoints-detection} \\
\path{spooky-author-identification} & \path{chaii-hindi-and-tamil-question-answering} & \path{statoil-iceberg-classifier-challenge} \\
\path{liberty-mutual-group-property-inspection-prediction} & \path{nbme-score-clinical-patient-notes} & \path{tgs-salt-identification-challenge} \\
\path{detecting-insults-in-social-commentary} & \path{20-newsgroups-ciphertext-challenge} & \path{global-wheat-detection} \\
\path{walmart-recruiting-store-sales-forecasting} & \path{word2vec-nlp-tutorial} & \path{whale-categorization-playground} \\
\path{prudential-life-insurance-assessment} & \path{jigsaw-toxic-comment-classification-challenge} & \path{dog-breed-identification} \\
\path{unimelb} & \path{text-normalization-challenge-english-language} & \path{plant-pathology-2020-fgvc7} \\
\path{nomad2018-predict-transparent-conductors} & \path{wsdm-cup-multilingual-chatbot-arena} & \path{dogs-vs-cats-redux-kernels-edition} \\
\path{amazon-employee-access-challenge} & \path{text-normalization-challenge-russian-language} & \path{petfinder-pawpularity-score} \\
\path{poker-rule-induction} & \path{llm-classification-finetuning} & \path{plant-seedlings-classification} \\
\path{allstate-purchase-prediction-challenge} & \path{stumbleupon} & \path{shopee-product-matching} \\
\path{dont-call-me-turkey} & \path{linking-writing-processes-to-writing-quality} & \path{the-nature-conservancy-fisheries-monitoring} \\
\bottomrule
\end{tabular}
\end{table}

\begin{table}[h]
\centering
\caption{\textbf{The 14 held-out competitions of MLE-Dojo (test).}}
\label{tab:heldout-tasks}
\setlength{\tabcolsep}{5pt}
\renewcommand{\arraystretch}{1.15}
\small
\begin{tabular}{>{\raggedright\arraybackslash}p{6.6cm} l >{\raggedright\arraybackslash}p{6.0cm}}
\toprule
\textbf{Competition} & \textbf{Category} & \textbf{Task} \\
\midrule
\path{GiveMeSomeCredit} & Tabular & Predict two-year default risk of borrowers \\
\path{forest-cover-type-kernels-only} & Tabular & Predict forest cover type from cartographic features \\
\path{novozymes-enzyme-stability-prediction} & Tabular & Rank thermostability of enzyme variants \\
\path{integer-sequence-learning} & Tabular & Predict the next term of integer sequences \\
\path{afsis-soil-properties} & Tabular & Predict five soil properties from infrared spectra \\
\path{quora-question-pairs} & Text & Decide whether two questions are duplicates \\
\path{AI4Code} & Text & Recover the order of markdown cells in notebooks \\
\path{jigsaw-unintended-bias-in-toxicity-classification} & Text & Detect toxicity while controlling identity bias \\
\path{quora-insincere-questions-classification} & Text & Flag insincere questions \\
\path{uw-madison-gi-tract-image-segmentation} & Image & Segment stomach and bowel in MRI scans \\
\path{invasive-species-monitoring} & Image & Detect an invasive plant species in photos \\
\path{kuzushiji-recognition} & Image & Detect and transcribe cursive Japanese characters \\
\path{bengaliai-cv19} & Image & Classify components of handwritten Bengali graphemes \\
\path{cassava-leaf-disease-classification} & Image & Classify cassava leaf diseases \\
\bottomrule
\end{tabular}
\end{table}

\paragraph{DSBench.}
The transfer suite is the data-modeling split of
DSBench~\citep{jing2024dsbench}, whose tasks are all drawn from Kaggle
(74 tasks per the paper, 75 task directories in the released data). We
keep the 60 tasks that run end to end in our sandbox. The other 15 are
excluded because the training data is too large for the sandbox (seven
tasks), the sample submission is missing (eight tasks), or the
submission format exceeds the scaffold's column limit (one task). The category columns of
Table~\ref{tab:main} contain 19 binary classification, 10 multi-class
classification, and 23 regression tasks, and the remaining 8 tasks
with non-standard metrics form the Other column. We also checked these 60 tasks against the 45 training competitions
with exact, normalized, and fuzzy name matching, and found no overlap.

\section{Experimental Details}
\label{app:experiment_details}

\paragraph{Training configuration.}
Both scales train with GRPO, eight groups per step and $n=8$
trajectories per group, with up to four interaction turns per
trajectory, at most 4096 generated tokens per turn, and observations
truncated to 1024 tokens. Optimization uses a constant learning rate
of $1\times10^{-6}$, a KL coefficient of $0.04$, and an entropy
coefficient of $0.002$, and rollouts are sampled at temperature $1.0$
with top-$p$ $1.0$. Real execution runs each solution in an
isolated environment with a 1200-second budget. The world model serves
the agent backbone with a 12k-token context and returns its prediction
within a 1024-token budget, and every step at least one group is
graded by real execution to feed the anchor pool.

\paragraph{Sandbox environment.}
Each solution executes in an isolated sandbox built as a Python 3.11
virtual environment on Linux, with a single \textasciitilde{}40\,GB
NVIDIA GPU and the path contract of the agent prompt, reading inputs
from \path{DATA_DIR} and writing the submission to
\path{SUBMISSION_PATH}. The environment holds 122 packages including
transitive dependencies, and installing the directly requested set of
Table~\ref{tab:sandbox} into a clean virtual environment reproduces
it.

\begin{table}[H]
\centering
\caption{\textbf{Libraries preinstalled in the execution sandbox}, by
category, with the exact versions.}
\label{tab:sandbox}
\setlength{\tabcolsep}{6pt}
\renewcommand{\arraystretch}{1.25}
\small
\begin{tabular}{l >{\raggedright\arraybackslash}p{0.66\textwidth}}
\toprule
\textbf{Category} & \textbf{Libraries (version)} \\
\midrule
Scientific computing & numpy 1.26.4, pandas 2.2.3, scipy 1.17.1,
statsmodels 0.14.6, sympy 1.14.0, networkx 3.6.1 \\
Classical ML & scikit-learn 1.4.2, xgboost 3.2.0, lightgbm 4.6.0,
catboost 1.2.10, joblib 1.5.3 \\
Deep learning (GPU) & torch 2.12.0 (CUDA 13.0), torchvision 0.27.0,
tensorflow 2.21.0, keras 3.14.1 \\
NLP & transformers 5.12.1, tokenizers 0.22.2, nltk 3.9.4 \\
Imaging and plotting & Pillow 12.2.0, matplotlib 3.10.9,
seaborn 0.13.2, plotly 6.6.0 \\
Utilities & h5py 3.14.0, tqdm 4.67.3, requests 2.34.2, regex \\
\bottomrule
\end{tabular}
\end{table}

\paragraph{Prompting the world model.}
\label{app:wm-prompt}
The world model is never fine-tuned and works by prompting alone. Its
prompt contains the same task description that the agent sees and the
agent's current solution, and it is instructed to simulate the
execution and report the outcome in the same format as the real
environment, from which the score is aggregated. The context is capped at
12k tokens and the prediction at 1024 tokens.
Appendix~\ref{app:prompts} sketches the structure of this prompt and
of the agent prompt.

\paragraph{Recalibration.}
The monotone map of Equation~\eqref{eq:isotonic} is the identity until
200 anchor pairs have accumulated, is first fit at that point, and is
refit after every 64 new pairs so that the correction follows the
drift of the world model.

\paragraph{Calibrating the reference disagreement.}
The disagreement $\hat\eta^{2}$ is normalized by the within-group
spread of the true scores, and this normalization is what lets a
reward-level statistic estimate a gradient-level ratio. The sampling
variance $V_{E}$ and the noise-injected excess of
$V_{WM}$ arise from the same estimator applied to the same
trajectories, so the conversion from score fluctuation to gradient
fluctuation largely cancels in their ratio, and what remains is the
noise variance measured against the natural spread of the scores,
which is $\hat\eta^{2}$ up to a conversion factor of order one shared
between the two streams. A single measurement at the end of warmup
fixes this factor, $c = 1/\hat\eta^{2}_{\mathrm{cal}}$, and we find the
calibrated value in one of our actual runs to be $0.96$, confirming that
it is indeed of order one. The fused variance is additionally
second-order insensitive to the weight (Lemma~\ref{lem:ivw}), so the
precision of this single measurement is not critical.

\paragraph{Bounded anchor weight.}
For training stability, we bound the anchor factor of
Equation~\eqref{eq:ivw} within $[1, w_{\max}]$ with $w_{\max}=4$,
which acts as a regularization on the fused gradient. The lower bound
keeps anchor groups from being down-weighted below the predicted
groups, and the upper bound prevents a transient spike of the measured
disagreement from letting a few anchor groups dominate a batch. The
tracked noise is likewise estimated over a recent window of the pool
$\mathcal{P}$ rather than the full history, so that the measurement
tracks the current policy.

\paragraph{VLA setup.}
The policy is MiniVLA-1B~\citep{belkhale2024minivla}, a
Prismatic-style model with a Qwen2.5-0.5B language backbone and a
vector-quantized action head that encodes an eight-step action chunk
into seven discrete tokens, pretrained on LIBERO-90. All models first
run supervised fine-tuning on the 50 official demonstrations per task,
and then GRPO for 40 steps. Each step samples 64 rollouts per task,
and each rollout takes at most 520 environment steps, the benchmark's
standard horizon. Training uses the first 16 official initial states of each
task while evaluation uses all 50, so most evaluated initial states
are never seen in training. Evaluation is greedy, with eight repetitions per initial state, since
the simulator's physics is stochastic.


\section{Prompts}
\label{app:prompts}

This appendix documents the interfaces of every model we prompt, in
both evaluation domains: the AutoResearch agent and its world model in
Appendix~\ref{app:prompts-ar}, and the VLA policy and its world model
in Appendix~\ref{app:prompts-vla}. The AutoResearch prompts are too
long to reproduce in full, so we sketch their structure, quote the
load-bearing passages verbatim, and mark every omission with
[\ldots]. The VLA interfaces are short enough to give in full.

\subsection{AutoResearch Agent}
\label{app:prompts-ar}

\paragraph{The agent prompt.}
Each turn, the AutoResearch agent sees a fixed system prompt, a user
prompt holding an automatically generated task overview, and the full
history of its previous attempts with their execution feedback. The
system prompt consists of the seven segments sketched below. The user
prompt is the first user message of the conversation, as the later
turns simply accumulate the interaction history on top of it, and the
task overview inside is generated by pure introspection over the task
files, with no manual per-task information. Since that overview is
what makes the prompt concrete, we instantiate it on one task and
carry the same task through the rest of this appendix, namely the
Kaggle competition \texttt{jigsaw-toxic-comment-classification\-challenge},
a six-label text classification task from our training pool. The
sections below are the real ones the builder emits, abridged with
[\ldots] where the full text is long.

\begin{promptbox}{System prompt of the AutoResearch agent}
\footnotesize
\textbf{Role and objective.}\\
\texttt{You are an expert ML engineer competing to MAXIMIZE your
leaderboard Position Score over UP TO K attempts (Position Score =
your percentile rank, higher = better; earned ONLY when your code runs
and writes a VALID submission). [\ldots]}

\medskip
\textbf{Two-phase strategy.}\\
\texttt{PHASE 1 --- LAND A VALID SUBMISSION FIRST (turn 1): write a
SIMPLE solution you are confident RUNS end-to-end and writes
SUBMISSION\_PATH in the exact required format. A valid submission is
your safety net. [\ldots]}\\
\texttt{PHASE 2 --- IMPROVE INCREMENTALLY (later turns): START FROM
YOUR LAST WORKING CODE, copy it, and change EXACTLY ONE thing per
turn. [\ldots]}\\
\texttt{If a change ERRORS or scores WORSE, discard it and try a
DIFFERENT single change. Never rewrite from scratch once something
works.}

\medskip
\textbf{Environment contract.}\\
\texttt{ENVIRONMENT CONTRACT (these Python variables are ALREADY
defined when your code runs): [\ldots]}\\
\texttt{DATA\_DIR: absolute path to the folder with the data files
listed below. [\ldots]}\\
\texttt{SUBMISSION\_PATH: absolute path to write your submission CSV
to. [\ldots]}\\
\texttt{A modern NVIDIA GPU (\textasciitilde{}40 GB) is available.
[\ldots]}\\
\texttt{Do NOT hardcode '/kaggle/...', '/content/...', or any other
path.}

\medskip
\textbf{Sandbox runtime.}\\
\texttt{Runtime: Python 3.11 on Linux (this is the sandbox where YOUR
code executes). Installed libraries you may import (name version):
numpy 1.26, pandas 2.2, scikit-learn 1.4, xgboost 3.2, torch 2.12
(CUDA), transformers 5.12 [\ldots]}\\
\texttt{Anything NOT in this list is not installed (e.g. no opencv, no
jax).}

\medskip
\textbf{Response format.}\\
\texttt{EACH TURN, respond in TWO parts, in this order: [\ldots]}\\
\texttt{1) ANALYSIS (2-4 short sentences of plain text, FIRST): name
exactly what went wrong and the ONE change that fixes it, or the ONE
change you will make to raise the score. [\ldots]}\\
\texttt{2) CODE: a SINGLE fenced Python block --- a complete,
self-contained script.}

\medskip
\textbf{Feedback loop.}\\
\texttt{After your code runs, the environment returns your print()
output and --- if a valid submission was written --- your leaderboard
Position and Raw Score; otherwise the full error traceback. [\ldots]}

\medskip
\textbf{Rules.}\\
\texttt{Keep a valid submission as your fallback; only replace it with
code that RUNS and scores BETTER. [\ldots]}\\
\texttt{The submission CSV must match the required format EXACTLY.
Keep code efficient so it finishes within the time limit.}
\end{promptbox}

\begin{promptbox}{User prompt of the AutoResearch agent}
\footnotesize
\textbf{Competition.}\\
\texttt{jigsaw-toxic-comment-classification-challenge}

\medskip
\textbf{Description} {\normalfont(head of the official competition
text, capped at 900 tokens; if the cap cuts off the evaluation
section, its head is force-appended)}\textbf{.}\\
\texttt{Build a multi-headed model that detects different types of
toxicity --- toxic, severe\_toxic, obscene, threat, insult, and
identity\_hate --- in Wikipedia talk-page comments. [\ldots]}\\
\texttt{Evaluation: mean column-wise ROC AUC over the six label
columns.}

\medskip
\textbf{Data inventory} {\normalfont(every file under
\texttt{DATA\_DIR}, with the full schema of each table)}\textbf{.}\\
\texttt{train.csv: 159571 rows $\times$ 8 cols: [id (str),
comment\_text (str), toxic (int), severe\_toxic (int), obscene (int),
threat (int), insult (int), identity\_hate (int)]}\\
\texttt{\phantom{0000}string-column examples -> id: e.g. '0000997932d777bf';
comment\_text: e.g. "Explanation Why the edits made"}\\
\texttt{test.csv: 153164 rows $\times$ 2 cols: [id (str),
comment\_text (str)]}\\
\texttt{sample\_submission.csv: 153164 rows $\times$ 7 cols: [id
(str), toxic (float), severe\_toxic (float), obscene (float), threat
(float), insult (float), identity\_hate (float)]}

\medskip
\textbf{Target columns} {\normalfont(present in train but not in
test)}\textbf{.}\\
\texttt{['toxic', 'severe\_toxic', 'obscene', 'threat', 'insult',
'identity\_hate']}

\medskip
\textbf{Submission contract} {\normalfont(written exactly this way to
\texttt{SUBMISSION\_PATH})}\textbf{.}\\
\texttt{columns = ['id', 'toxic', 'severe\_toxic', 'obscene',
'threat', 'insult', 'identity\_hate']}\\
\texttt{FILL the prediction column(s) with your model output:
['toxic', \ldots, 'identity\_hate']. COPY the column(s) ['id']
straight from the test set (id / passthrough).}\\
\texttt{Example rows (cells truncated for display):}\\
\texttt{\phantom{0000}00001cee341fdb12~~0.5~~0.5~~0.5~~0.5~~0.5~~0.5}\\
\texttt{\phantom{0000}0000247867823ef7~~0.5~~0.5~~0.5~~0.5~~0.5~~0.5}

\medskip
\textbf{Trigger.}\\
\texttt{Begin.}
\end{promptbox}

\paragraph{The world model prompt.}
The world model shares one simulator system prompt across eight
parallel single-aspect checks, which cover imports and names, API
calls against the installed versions, file and column references
against the data inventory, shapes and dtypes, the compute budget,
remaining runtime errors, submission validity, and solution quality.
All eight checks share the same system prompt. The user prompt of each
check holds its specific instruction followed by the case block,
namely the task overview above and the agent's code with 1-indexed
line numbers, and the check returns a verdict in strict JSON. The
verdicts are then aggregated into the scalar score $\hat r$, where
failures cap the score and a valid solution is scored by the quality
check. The user prompt below is that of the file and column reference
check, instantiated on the same task as above and on a real turn-one
attempt at it. That attempt assigns the literal strings
\texttt{"DATA\_DIR"} and \texttt{"SUBMISSION\_PATH"} to the two
variables instead of reading them from the environment, so the first
\texttt{read\_csv} resolves to a path that does not exist. This is
exactly the failure the check is meant to catch, and we print the
verdict it returns underneath the prompt.

\begin{promptbox}{System prompt of the world model}
\footnotesize
\textbf{Role.}\\
\texttt{You are a precise execution-and-grading SIMULATOR for
Kaggle/MLE Python solutions. You do NOT execute code. You predict ---
by careful static analysis --- exactly what the real sandbox would
report.}

\medskip
\textbf{Environment brief.}\\
\texttt{Sandbox: Python 3.11 on Linux, ONE \textasciitilde{}40GB GPU
available. Installed (name version): numpy 1.26, pandas 2.2
[\ldots]}\\
\texttt{Execution is killed at 1200s. The code reads data from
os.environ['DATA\_DIR'] and MUST write predictions to
os.environ['SUBMISSION\_PATH']. [\ldots]}

\medskip
\textbf{Inputs.}\\
\texttt{You are given the TASK OVERVIEW (with a DATA INVENTORY: real
files, columns, row counts, dtypes, and the sample\_submission format)
and the agent's CODE (its lines are 1-indexed as shown). Analyze ONLY
the failure mode this check asks about.}

\medskip
\textbf{Calibration policy.}\\
\texttt{CRITICAL --- DEFAULT TO PASS. Real submissions usually RUN
FINE; most code does NOT trigger this check's failure mode. [\ldots]}\\
\texttt{Return verdict='fail' ONLY when you can cite a SPECIFIC line
and a CONCRETE, CERTAIN reason it raises this exact error. [\ldots]}\\
\texttt{A false 'fail' on working code is WORSE than a missed error.
When in doubt, PASS.}

\medskip
\textbf{Error specificity.}\\
\texttt{If you do fail it, be concrete: identify the EXACT offending
source line, its line number, the precise Python exception class, and
the real error message. [\ldots]}\\
\texttt{The agent will READ your env\_feedback to fix its code next
turn, so env\_feedback MUST look byte-similar to the real sandbox
output. [\ldots]}

\medskip
\textbf{Output format.}\\
\texttt{Output STRICT JSON only, no prose around it:}\\
\texttt{\{"reason": "<= 2 sentences, the specific cause>", "verdict":
"<pass|fail>", "confidence": <float 0..1>, "error\_type": [\ldots],
"error\_line": [\ldots], "env\_feedback": [\ldots]\}}
\end{promptbox}

\begin{promptbox}{User prompt of the world model}
\footnotesize
\textbf{Check instruction.}\\
\texttt{CHECK: Will this code raise FileNotFoundError or
KeyError(column) --- i.e. does it reference a file/path/column that
does not match the DATA INVENTORY? [\ldots]}

\medskip
\textbf{Decision rule.}\\
\texttt{verdict=fail if a FileNotFoundError or KeyError is certain;
env\_feedback = the exact message.}

\medskip
\textbf{Task.}\\
\texttt{jigsaw-toxic-comment-classification-challenge}

\medskip
\textbf{Task overview and data inventory} {\normalfont(the same
overview the agent received, reproduced verbatim)}\textbf{.}\\
\texttt{Build a multi-headed model that detects different types of
toxicity [\ldots]}\\
\texttt{train.csv: 159571 rows $\times$ 8 cols: [id (str),
comment\_text (str), toxic (int), \ldots] [\ldots]}

\medskip
\textbf{Agent code} {\normalfont(the attempt under review, lines
1-indexed)}\textbf{.}\\
\texttt{\phantom{00}1~~import os}\\
\texttt{\phantom{00}2~~import pandas as pd}\\
\texttt{\phantom{00}3~~import numpy as np}\\
\texttt{\phantom{00000}[\ldots]}\\
\texttt{\phantom{00}9~~\# Paths (using the provided environment variables)}\\
\texttt{\phantom{0}10~~DATA\_DIR = "DATA\_DIR"}\\
\texttt{\phantom{0}11~~SUBMISSION\_PATH = "SUBMISSION\_PATH"}\\
\texttt{\phantom{0}12~}\\
\texttt{\phantom{0}13~~\# Load data}\\
\texttt{\phantom{0}14~~train\_df = pd.read\_csv(os.path.join(DATA\_DIR, "train.csv"))}\\
\texttt{\phantom{0}15~~test\_df = pd.read\_csv(os.path.join(DATA\_DIR, "test.csv"))}\\
\texttt{\phantom{00000}[\ldots]}
\end{promptbox}

\begin{promptbox}{Verdict returned by the world model}
\footnotesize
\texttt{\{}\\
\texttt{\phantom{00}\textbf{"reason"}: "DATA\_DIR is the literal string
\textbackslash{}"DATA\_DIR\textbackslash{}" instead of
os.environ['DATA\_DIR'], so read\_csv looks for 'DATA\_DIR/train.csv'
which doesn't exist.",}\\
\texttt{\phantom{00}\textbf{"verdict"}: "fail",}\\
\texttt{\phantom{00}\textbf{"confidence"}: 0.95,}\\
\texttt{\phantom{00}\textbf{"error\_type"}: "FileNotFoundError",}\\
\texttt{\phantom{00}\textbf{"error\_line"}: 14,}\\
\texttt{\phantom{00}\textbf{"env\_feedback"}: "=== Code Execution
Results ===\textbackslash{}nExecution failed: Traceback (most recent
call last):\textbackslash{}n~~File
\textbackslash{}"/tmp/tmph32t8gzr.py\textbackslash{}", line 28, in
<module>\textbackslash{}n~~~~train\_df =
pd.read\_csv(os.path.join(DATA\_DIR,
\textbackslash{}"train.csv\textbackslash{}")) [\ldots]"}\\
\texttt{\}}
\end{promptbox}

\subsection{VLA Agent}
\label{app:prompts-vla}

\paragraph{The agent prompt.}
At every control step the VLA agent is prompted from scratch, as the
policy is memoryless and no interaction history accumulates on top of
the first message. MiniVLA-1B is built on the \texttt{openvla-mini}
(Prismatic) codebase and we reuse its prompt builder unchanged, so the
prompt is the stock Qwen-2 chat template of the backbone, with one
system turn and one user turn. We add no prompt engineering of our
own: the system turn is the stock backbone line, and the user turn
carries only the current camera image and the task instruction,
lower-cased. Keeping the template byte-identical to the upstream one
matters because the checkpoint we start from is pretrained behind
exactly this template, so any deviation in wording, casing, or the
chat scaffold would move the policy off its pretraining distribution
before RL even starts. We instantiate it below on task 8 of
LIBERO-Long, \emph{put both moka pots on the stove}, and carry that
task through the rest of this subsection.

\begin{promptbox}{System prompt of the VLA agent}
\footnotesize
\textbf{Role} {\normalfont(the stock backbone line, unchanged)}\textbf{.}\\
\texttt{<|im\_start|>system}\\
\texttt{You are Qwen, created by Alibaba Cloud. You are a helpful
assistant.<|im\_end|>}
\end{promptbox}

\begin{promptbox}{User prompt of the VLA agent}
\footnotesize
\textbf{Visual observation} {\normalfont(the third-person agentview
frame the simulator renders at the current step, resized to
$224\times224$; it enters the sequence as projected patch embeddings,
not as text)}\textbf{.}\\[3pt]
\hspace*{1em}\includegraphics[width=2.1cm]{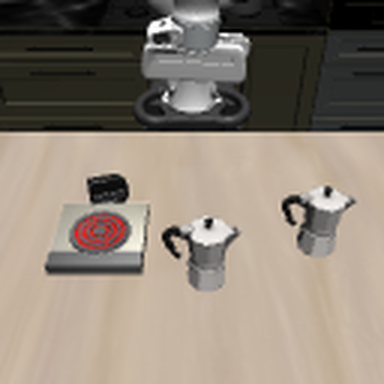}

\medskip
\textbf{Instruction.}\\
\texttt{<|im\_start|>user}\\
\texttt{What action should the robot take to put both moka pots on the
stove?<|im\_end|>}\\
\texttt{<|im\_start|>assistant}

\medskip
\textbf{Response format.}\\
\texttt{<exactly 7 discrete action tokens, one code from each of the 7
residual codebooks of the action head; no free-form text is
generated>}
\end{promptbox}

\paragraph{The world model prompt.}
The world model of this domain is Robometer-4B, a full fine-tune of
Qwen3-VL-4B-Instruct on robot-progress data. Unlike the AutoResearch
world model it is never asked to simulate an execution it has not run:
it reads the frames the policy actually produced and estimates how far
the task has progressed at each of them. It also carries no system
prompt at all. One rollout becomes one single-turn conversation, whose
user message holds the instruction text below followed by the frames
the policy observed, uniformly subsampled to eight, the window the
model was trained on. The reply is a per-frame progress value together
with a per-frame success probability.

\begin{promptbox}{User prompt of the world model}
\footnotesize
\textbf{Instruction} {\normalfont(verbatim; the task string is
inserted)}\textbf{.}\\
\texttt{The task for the robot is 'put both moka pots on the stove'.
Given the trajectory video, predict the task progress at each frame,
how far along the robot is towards completing the task, a float
between 0 and 1, where 0 is the starting state and 1 is when the task
is completed. If the robot is not performing the same task, predict 0
progress.}

\medskip
\textbf{Rollout frames} {\normalfont(8 frames uniformly subsampled
from the rollout; 4 of them shown, spanning start to end of the
episode)}\textbf{.}\\[3pt]
\hspace*{1em}\includegraphics[width=1.9cm]{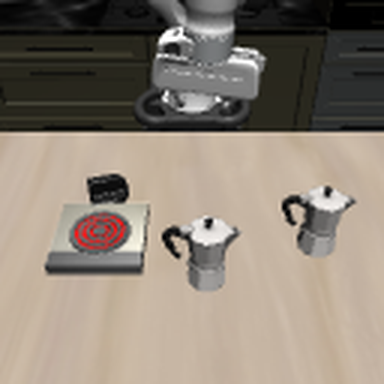}\hspace{5pt}%
\includegraphics[width=1.9cm]{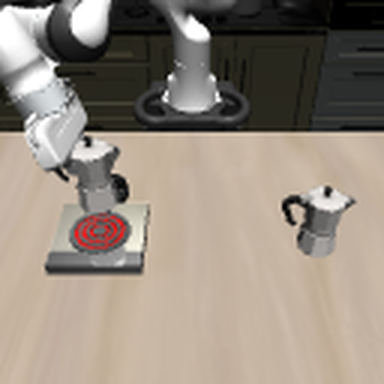}\hspace{5pt}%
\includegraphics[width=1.9cm]{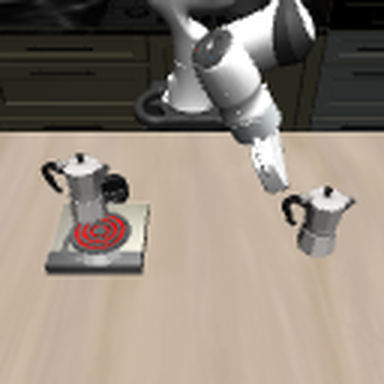}\hspace{5pt}%
\includegraphics[width=1.9cm]{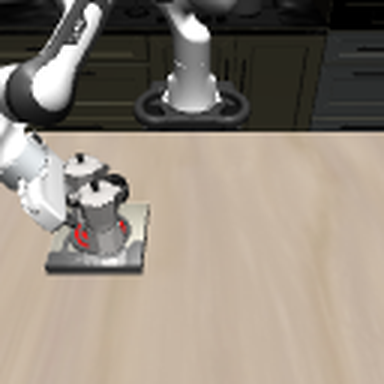}

\medskip
\textbf{Response format.}\\
\texttt{<a progress value in [0,1] and a success probability in [0,1]
for each of the 8 frames, read out by the discrete progress head and
the success head>}
\end{promptbox}

\end{document}